\documentclass[pre,preprint,nofootinbib,floatfix]{revtex4-1}

\usepackage{nicefrac}       
\usepackage{verbatim}
\usepackage{microtype}      
\usepackage{amsmath,amssymb,amsfonts,amsthm,amscd,bm}
\usepackage{algorithm}
\usepackage{algorithmic}
\usepackage[inline]{enumitem}
\usepackage[pdftex]{graphicx}
\graphicspath{{./figs/}}
\usepackage{multirow}
\usepackage{colortbl}
\usepackage{array,booktabs}
\usepackage{xcolor}
\usepackage{makecell}
\usepackage{pifont}
\newcommand{\cmark}{\ding{51}}
\newcommand{\xmark}{\ding{55}}

\newtheorem{theorem}{Theorem}

\newtheorem{lemma}[theorem]{Lemma}
\newtheorem{corollary}[theorem]{Corollary}

\DeclareMathOperator*{\argmin}{arg\,min}

\DeclareMathOperator{\diag}{diag}

\newcommand{\G}{\mathcal{G}}

\newcommand{\E}{\mathcal{E}}
\newcommand{\V}{\mathcal{V}}
\newcommand{\bG}{\bar{\mathcal{G}}}
\newcommand{\bF}{\bar{\mathcal{F}}}
\newcommand{\bE}{\bar{\mathcal{E}}}
\newcommand{\bV}{\bar{\mathcal{V}}}

\newcommand{\D}{\mathcal{D}}
\newcommand{\W}{\mathcal{W}}

\renewcommand{\AA}{\mathcal{A}}
\renewcommand{\L}{L}
\newcommand{\LL}{\mathcal{L}}
\newcommand{\M}{M}

\newcommand{\MC}{\mathcal{M}}
\newcommand{\hMC}{\hat{\mathcal{M}}}
\newcommand{\HH}{\mathcal{H}}
\newcommand{\hHH}{\hat{\mathcal{H}}}

\newcommand{\B}{B}
\newcommand{\RR}{R}
\newcommand{\OO}{\Omega}
\newcommand{\tB}{\widetilde{\B}}

\newcommand{\vx}{\bm{x}}
\newcommand{\vz}{\bm{z}}

\newcommand{\vu}{\bm{u}}
\newcommand{\vm}{\bm{m}}
\newcommand{\vn}{\bm{n}}
\newcommand{\vs}{\bm{s}}
\newcommand{\vv}{\bm{v}}
\newcommand{\vy}{\bm{y}}

\newcommand{\bomega}{\bar{\omega}}
\def\bell{\bar{\ell}}

\begin{document}

\title{How is Distributed ADMM Affected by Network Topology?}

\author{Guilherme Fran\c ca}
\email{guifranca@gmail.com} 
\affiliation{Boston College, Computer Science Department}
\affiliation{Johns Hopkins University, Center for Imaging Science}
  
\author{Jos\' e Bento}
\email{jose.bento@bc.edu}
\affiliation{Boston College, Computer Science Department}

\begin{abstract}
When solving consensus optimization problems over a graph, there is 
often an explicit characterization of the convergence rate of 
Gradient Descent (GD) using the spectrum of the graph Laplacian.
The same type of problems under the Alternating Direction Method
of Multipliers (ADMM) are, however, poorly understood. 
For instance, simple but important non-strongly-convex
consensus problems have not yet being analyzed, 
especially concerning the dependency
of the convergence rate on the graph topology.
Recently, for a non-strongly-convex consensus problem, 
a connection between distributed ADMM and lifted Markov chains
was proposed,
followed by a conjecture that ADMM 
is faster than GD by a square root factor in its convergence time, 
in close analogy to the mixing
speedup
achieved by lifting several Markov chains. 
Nevertheless, a proof of such a claim is is still lacking.
Here we provide a full characterization of the convergence of  
distributed over-relaxed
ADMM for the same type of consensus problem in terms of the topology 
of the underlying graph. Our results provide explicit formulas for
optimal parameter selection in terms of the second largest eigenvalue
of the transition matrix of the graph's random walk.
Another consequence of our results is a proof of the 
aforementioned conjecture,
which interestingly, we show it is valid for any graph, 
even the ones whose random walks cannot be 
accelerated via Markov chain lifting.
\end{abstract}

\maketitle

\section{Introduction}

Optimization methods are at the core
of statistics and machine learning.
In this current age of ever-larger datasets, traditional in-memory 
methods do not scale, so distributed algorithms play a fundamental role.
The Alternating Direction Method of Multipliers (ADMM) 
is one such excellent example 
since it is extremely robust, for instance does not assume differentiability 
of the objective function, it is often easy to implement, 
and easily distributed \cite{Boyd}. Moreover, 
ADMM attains global linear convergence for separable and convex functions
\cite{Hong2017}, and 
is guaranteed to converge even for several 
non-convex problems \cite{Wang}, and empirically for 
many others \cite{Bento1,Bento2,Tom}. 
Nevertheless, its convergence rate is still, in general, not 
fully understood. Most existing results
only provide upper bounds on its asymptotic convergence rate without
tightness guarantees.
For more precise results, 
strong-convexity is usually assumed, even in centralized settings
\cite{FrancaBento,Jordan,giselsson2017linear, Deng2016}.
Among practitioners ADMM also has a fame of being hard to tune.

In this paper we analyze how
the \emph{exact and optimally tuned} asymptotic convergence 
rate of a distributed 
implementation of 
over-relaxed ADMM depends on the topology of an underlying network. Through this network, 
several agents solving local problems share messages to one another
with the common goal of solving a large optimization problem. 
One of our motivations is to understand, in a quantitative way, 
if ADMM is more or less sensitive to 
the network topology than distributed 
Gradient Descent (GD).
We focus on a non-strongly-convex quadratic consensus problem 
not previously analyzed under ADMM.

Since the convergence
rate of the algorithm 
may be dominated by many properties of the objective function, such as 
its curvature, and since
our goal is
to focus only on
the topology of the network, we 
choose
an objective function that emphasizes how variables
are shared among its terms.
Consider an undirected, connected, and simple graph
$\G = (\V, \E)$,
where
$\V$ is the set of vertices and $\E$ the set of edges. Let $\vz \in
\mathbb{R}^{|\V|}$ be the set of variables, 
where $z_i \in \mathbb{R}$ denotes the $i$th component of $\vz$ and
is associated to node $i \in \V$.
We study the following consensus problem over $\G$:
\begin{equation} \label{eq:gen_quad}
\min_{ \vz \in \mathbb{R}^{|\V|}} \bigg \{ 
f(\bm{z}) = \dfrac{1}{2} \sum_{(i,j)  \in \E} (z_i - z_j)^2 
\bigg \} .
\end{equation}
Our goal is to provide a precise answer on
how the convergence rate  of ADMM when 
solving problem \eqref{eq:gen_quad}
depends on properties of $\G$. 
We also want to compare the convergence rate of ADMM 
with the convergence rate of GD when
solving the same consensus problem.

The optimization problem \eqref{eq:gen_quad} is deceptively simple, having
the trivial solution $x_i = x_j$ if $i$ and $j$ belong to the
same connected component of $\G$. However, 
it is not immediately obvious 
to which of these infinitely many possible solutions a given distributed 
algorithm will converge to.
Different agents of the distributed implementation 
have to communicate to agree on the final solution, 
and the speed at which they reach consensus is a non-trivial problem. 
For instance, if we solve \eqref{eq:gen_quad} through ADMM we have one agent 
per term of the objective function, and each agent has 
local copies of all the variables involved. 
The final solution is a vector where each component equals 
the average of the initial values of these local variables. 
Therefore, unsuspectingly, we have solved a distributed-average 
consensus problem, although in different form than typically studied, see e.g. \cite{erseghe2011fast,Ghadimi2014averaging}.
Moreover, the objective function \eqref{eq:gen_quad} naturally
appears in several interesting problems. A classical example is 
the graph interpolation problem \cite{zhu2003semi}, where one 
solves \eqref{eq:gen_quad} subject to $z_i = c_i$ for $i \in \V'$ where
$\V' \subset \V$ and $c_i$ is a fixed constant.
The final solution
has values on each node of $\G$ such that the nodes in $\V'$
have the pre-assigned values, and the remaining nodes in $\V \backslash \V'$
have values close to the values of their neighboring nodes 
depending on the topology of $\G$. Our analysis of ADMM to \eqref{eq:gen_quad}
may provide insights into other problems such as graph interpolation.
Furthermore, it can give insights on how one can optimally split
a decomposable objective function for a given optimization problem.

Let us formalize our problem.
Define the asymptotic convergence rate, $\tau$, of
an algorithm by
\begin{equation}\label{eq:conv_def}
\log \tau \equiv \lim_{t\to \infty} 
\max_{\| \vz^0 \| \le 1} \left\{
\dfrac{1}{t} 
\log \| \vz^{\star} - \vz^t  \|\right\},
\end{equation}
where $t$ is the iteration time, and $\vz^\star$ is a minimizer of
\eqref{eq:gen_quad} that the iterate $\vz^t$ converges to.
Denote $\tau_A$ and $\tau_G$ be the convergence rates of
ADMM and GD, respectively. We want to obtain the dependence
$\tau_A = \tau_A(\G)$ and $\tau_G = \tau_G(\G)$, and also be able
to compare the optimal rates,
$\tau^\star_A$ and $\tau^\star_G$,
when the parameters of both algorithms are optimally chosen.

The present work is motivated by an
interesting idea recently proposed in \cite{FrancaBentoMarkov},
which relates distributed ADMM to \emph{lifted} Markov chains.
It was shown that
\begin{enumerate*}[label=(\roman*)]
\item ADMM is related to a quasi-Markov chain $\hMC$,
\item  GD is related to a Markov chain $\MC$, and
\item $\hMC$ is a lifting of $\MC$.
\end{enumerate*}
In general, a lifted Markov chain $\hMC$ is obtained from
the base Markov chain $\MC$ by expanding its state space in 
such a way that
it is possible to \emph{collapse}
$\hMC$ into $\M$ and $\hat{\pi}$ into $\pi$, where $\hat{\pi}$ and $\pi$ are
the respective
stationary distributions (see \cite{chen1999lifting} for details).
The hope is that if $\MC$ is slow mixing 
one can sample from $\pi$ 
by collapsing samples from $\hat{\pi}$, where $\hMC$ mixes faster than $\MC$.
A measure of the time required for $\MC$ to reach stationarity
is given by the mixing time, denoted by $\HH$.
For many useful cases, the mixing time of the lifted chain, 
$\hHH$,
is smaller than $\HH$.
However, the achievable speedup is limited. 
For example, if $\MC$ is irreducible then
$\hHH \ge C \sqrt{\HH}$ for some constant $C \in (0,1)$, and
there are several cases that actually
achieve the lower bound,
$\hHH \approx C \sqrt{\HH}$.
The gain can be marginal if both
$\MC$ and $\hMC$ are reversible, where one has $\hHH \ge C \HH$. 
Furthermore,
for some graphs, for example graphs with \emph{low conductance},
lifting never produces a significant speedup.

In \cite{FrancaBentoMarkov} the quantity $(1 - \tau_A)^{-1}$ plays the role
of $\hHH$, while $(1-\tau_G)^{-1}$ plays the role of $\HH$. 
Based on the lifting
relations between ADMM and GD, and 
the many cases where $\hHH \approx C \sqrt{\HH}$, it was conjectured that 
\begin{equation}\label{eq:conjecture}
1-\tau_A^\star \ge C \sqrt{1 - \tau_G^\star}
\end{equation}
for some constant $C$. Above, 
$\tau^\star$ denotes the optimal convergence rate,
attained under optimal parameter selection.
It is important that both algorithms are optimally tuned 
since a poorly tuned ADMM can be slower than a well-tuned GD. 
The inequality
\eqref{eq:conjecture}
was supported by empirical evidence but its proof remains lacking.
As pointed out \cite{FrancaBentoMarkov}, 
the inequality \eqref{eq:conjecture} is much stronger than the analogous 
relation in lifted Markov chains theory. The
claim is that it holds for \emph{any}
graph $\G$, even the ones whose Markov chains do not accelerate
via lifting.

The rest of the paper is organized as follows.
In Section~\ref{sec:related_work} we mention related work and point
out the differences and novelty in our approach.
In Section~\ref{sec:admm_message} we introduce important notation and
concepts by explicitly writing distributed over-relaxed ADMM as
a message passing algorithm.
We then present our main contributions, which in short
are:
\begin{enumerate*}[label=(\roman*)]
\item in Section \ref{sec:admm_spectrum},
we prove a relation between the spectrum of a nonsymmetric 
matrix related to the evolution of ADMM and the spectrum of the 
transition matrix of a random walk on $\G$. This relates
ADMM to random walks on $\G$, capturing the topology of the
graph. \item We also prove explicit formulas for optimal parameter
selection, yielding interesting relations to the second largest
eigenvalue of the transition matrix of the graph $\G$ and the spectral gap.
\item 
In Section \ref{sec:admm_vs_gd}, we resolve the conjectured 
inequality \eqref{eq:conjecture}, and moreover, provide an upper bound.
\end{enumerate*}
The proofs of our main results are in the Appendix.

\section{Related Work}
\label{sec:related_work}

Although problem \eqref{eq:gen_quad} is simple
our results cannot be directly derived
from any of the many existing results on the convergence of ADMM.
First of all, we compute the exact
asymptotic convergence rate when distributed ADMM is optimally tuned, while
the majority of previous works only compute non-optimal upper bounds for
the global convergence rate.
Second, our convergence rate is linear, and most works able
to prove tight linear convergence assume strong convexity 
of at least some of the functions in the objective; see for instance
\cite{shi2014linear,giselsson2014diagonal,giselsson2017linear,Deng2016}. 
It is unclear if we can cast our non-strongly-convex problem in 
their form and recover our results from their bounds, given
especially that most of these results are not simple or explicit 
enough for our purposes. 
These bounds often have a complex dependency on problem's parameters,
but can be numerically optimized 
as suggested by \cite{giselsson2014diagonal,giselsson2017linear}. 
It is also unknown if these numerical procedures lead to optimal
rates of convergence.
Linear convergence rates were proven without strong convexity 
\cite{Hong2017}, but these bounds are too general and not tight 
enough for our purposes.
Moreover, many results not requiring strong convexity focus on 
the convergence rate of the objective function, as opposed to this paper
which studies the convergence rate of
the variables; see for example \cite{davis2014convergence, davis2017faster}.


In \cite{erseghe2011fast,Ghadimi2014averaging}  
ADMM is applied to the
consensus problem
$f(\vz) = \sum_{i \in \V }\sum (z_i - c_i)^2$,
subject to  $z_i = z_j$ if $(i,j) \in \E$, where $c_i > 0$ are
constants. This problem, which is related to
several optimization-based distributed averaging algorithms, is strongly-convex and not
equivalent to
\eqref{eq:gen_quad}.
Several papers consider $f(\vz) = \sum_i f_i(\vz)$ with
ADMM updates that are insensitive to
whether or not $f_i(\vz)$ depends on a subset
of the components of $\vz$; see
\cite{wei2012distributed, ling2015dlm, 
makhdoumi2014broadcast,makhdoumi2017convergence,
ling2016communication} and references therein.
In our
setting,
distributed ADMM is a message-passing algorithm where
the messages between agents $i$ and $j$
are related only to the variables shared by functions
$f_i$ and $f_j$. Thus, our implementation is fully local, and not only the
processing but also the data is distributed. 
These papers solve $\min_{\bf z} \sum_i f_i(\vz)$ 
over a communication network
by recasting the problem as  $\min \sum_i f_i(\vx_i)$
subject to $\vx_i = \vx_j$ if $(i,j)$ are edges in the network.
Slight variations of this transformation and the 
definition  of the communication
network exist.
Several of these works try to understand how topology of the network
affects convergence, for instance
\cite{ling2015dlm,makhdoumi2014broadcast,makhdoumi2017convergence,
ling2016communication}.
The results of 
\cite{makhdoumi2014broadcast,makhdoumi2017convergence}
are applicable to non-strongly-convex objectives but
linear convergence rates are not proven.
An interesting 
adaptive ADMM for a general convex consensus problem was recently proposed
\cite{TomConsensus}, with a sublinear convergence guarantee. However,
no dependence on the underlying graph was considered.
It is important to note that, even for GD, the dependency of
the convergence rate on $\G$ for variants 
of problem \eqref{eq:gen_quad} have only 
being studied in the past 
decade \cite{olfati2004consensus, tron2008distributed, 
tron2013riemannian}.

For quadratic problems
there are explicit results on convergence rate and
optimal parameters 
\cite{teixeira2013optimal,ghadimi2015optimal,iutzeler2014linear,
iutzeler2016explicit}.
However, the required assumptions do not hold
for the distributed consensus problem considered in this paper.
Moreover, there are very few results comparing
the optimal convergence rate of ADMM as a function of
the optimal convergence rate of GD.
An explicit comparison is
provided in \cite{FrancaBento}, but assumes
strong convexity and considers a centralized setting.
The authors in \cite{mokhtari2015decentralized} study a variant
of the ADMM where
the iterations deal with the second order expansion of the
objective, and strong-convexity is assumed.
In \cite{raghunathan2014optimal,raghunathan2014admm} bounds
on the convergence rate were proven, which are subsequently
tuned for ADMM applied to a quadratic program of the kind
$\min_{\bf z} {\bf z}^\top Q{\bf z} + c^\top {\bf z}$ 
subject to $A{\bf z} = b$, and also
assume strong convexity. 
The work \cite{boley2013local} focuses on quadratic programs 
that are not necessarily strongly-convex. To the best of our knowledge,
this is the only work that, just like we do here, analyzes ADMM for 
quadratic problems in a setting where the important eigenvalues 
of the transition matrix might be complex numbers. 
However, no optimal bounds explicitly dependent on
$\G$ are provided. The authors of \cite{han2013local} also study  
ADMM for quadratic programs that might not be strongly-convex. 
They define their error rate in a different way compared to us, 
and it is not clear if they are comparable. 
Also, their bounds are generic and are not optimally tuned. 

The problem of determining optimal rates of convergence is related 
to 
optimal parameter selection. Apart from the tuning rules mentioned above,
several adaptive 
schemes exist, and 
some of these come with convergence guarantees
\cite{he2000alternating,Boyd,xu2016adaptive,xu2017adaptive}.
However, these are designed for very general problems 
and do not recover our results. 
We consider ADMM's parameters fixed across iterations.

Our work makes connections between ADMM, GD, and Markov chains. 
In particular,
lifted Markov chains were previously employed
to speedup convergence of
distributed averaging and gossip algorithms
\cite{jung2007fast, li2010location, jung2010distributed}, 
but these algorithms are not related to ADMM.
Finally, the present work is highly motivated by \cite{FrancaBentoMarkov} 
where
a close relation between ADMM and lifted Markov chains was proposed.
The main outcome was
conjecture \eqref{eq:conjecture}, which
is inspired by the speedup on the mixing time of several lifted Markov
chains. This inequality will be proven
in this paper as a consequence of our main analysis.

\section{Distributed ADMM as a Message Passing Algorithm}
\label{sec:admm_message}

Let us start by introducing the factor graph 
$\bG$ associated to the base graph $\G$ of problem \eqref{eq:gen_quad}. 
The factor graph $\bG = (\bF, \bV, \bE)$ is a bipartite and undirected 
graph, where the edges in $\bE$ can only connect vertices 
in $\bF$ to vertices in $\bV$. 
The $a$th vertex in $\bF$ is the $a$th term $f_a$ in the objective
\eqref{eq:gen_quad}. 
In other words, we have a function vertex $f_a$ for every edge in $\E$.  
Vertices in $\bF$ are called function nodes. The $b$th vertex in $\bV$ is  
the $b$th component $z_b$ of $\vz$.  We have a variable vertex per 
dimension of $\vz$. Vertices in $\bV$ are called variable nodes. 
The edges in $\bE$ are of the form $(f_a, z_b)$.
The crucial point in defining $\bG$ is that 
the edge $(f_a, z_b)$ is present in $\bE$ if and only if
$f_a$ depends on the $z_b$ component of $\vz$. To simplify the notation, 
sometimes we  interchangeably refer to vertices and edges only by 
their labels, 
thus we might write $(f_a, z_b) = (a, b)$ with $a \in \bF$ and $b \in \bV$.  
Therefore, $\bV = \V$ and $|\bE| = 2 |\E|$. We refer 
to Fig.~\ref{fig:alice} for an illustration.

The neighborhood a  
function node $a$ is denoted by
$N_a \equiv \{ b \in \bV : (a, b) \in \bE \}$.
Analogously, the neighborhood a variable node $b \in \bV$
is $N_b \equiv \{ a \in \bF : (a, b ) \in \bE \}$.
For example, in the case of Figure~\ref{fig:alice}b we have
$N_a = \{ z_b, z_d \}$ and $N_d = \{ f_a, f_c\}$.
Let us introduce the row stochastic matrix
$S \in \mathbb{R}^{|\bE|\times |\bV|}$ 
defined by
\begin{equation}\label{eq:Sdef}
S_{eb} = \begin{cases}
1 & \mbox{if $e \in \bE$ is incident on $b \in \bV$,} \\
0 &  \mbox{otherwise.}
\end{cases}
\end{equation}
The action of $S$ on a vector $\vz \in \mathbb{R}^{|\bV|}$ is
to produce an $|\bE|$-dimensional vector whose
$e$th component, for $e \in \bE$, is equal to $z_b$ if
$e$ is incident on $b \in \bV$.
Through the paper, we often index
the components of a vector
$\bm{y} \in \mathbb{R}^{|\bE|}$ by 
the edges of the factor graph, such as
$y_{ab}$, where $e=(a,b) \in \bE$. 
For any vector $\bm{w} \in \mathbb{R}^{|\bV|}$,
we often index its components by
the variable nodes as
$w_b$, where $b \in \bV$ (see Fig.~\ref{fig:alice}b).

Now let $\vx \in \mathbb{R}^{|\bE|}$ with components $x_{ab}$.
For each function node $a \in \bF$ we define the vector
$\vx_a\in\mathbb{R}^{|N_a|}$ with components 
in $\{ x_{ab}:  b \in N_a \}$.
We can rewrite problem \eqref{eq:gen_quad} by introducing
the decoupled objective
\begin{equation}\label{eq:quadratic}
f(\vx) = \dfrac{1}{2} \vx^\top Q \, \vx = 
\sum_{a \in \bF} f_a(\vx_a) = 
\dfrac{1}{2} \sum_{a \in \bF} \vx_a^\top Q_a \, \vx_{a},
\end{equation}
where $Q$ is a block diagonal matrix with blocks in the form 
$Q_a = \left( \begin{smallmatrix} +1 & -1 \\ -1 & +1 
\end{smallmatrix}\right)$, and
adding the constraint
\begin{equation}
\vx = S \vz.
\end{equation}
The idea is that the ADMM can exploit
this decoupled objective function and solve problem \eqref{eq:gen_quad}
in a distributed manner, by coordinating local messages 
that are computed only based on each $f_a$.


We can decentralize the standard over-relaxed ADMM updates \cite{Boyd} 
with the help of the so-called \emph{message passing} variables
$\vm \in \mathbb{R}^{|\bE|}$ and 
$\vn \in \mathbb{R}^{|\bE|}$, and the \emph{dual variable} 
$\vu \in \mathbb{R}^{|\bE|}$:
%
\begin{subequations} 
\label{eq:admm_msg_updates}
\begin{align} 
\vx_{a}^{t+1}  & \leftarrow \argmin_{\vx_{a}} 
\bigg\{ f_a\left( \vx_a \right) + 
\dfrac{\rho}{2} \sum_{b \in N_a} \left( x_{ab} - n_{ab}^t \right)^2 \bigg\} 
& &\mbox{for all $a \in \bF$},
\label{eq:msg1} \\
m_{ab}^{t+1} &\leftarrow \gamma x_{ab}^{t+1} + u_{ab}^t 
& &\mbox{for all $(a,b) \in \bE$} ,
\label{eq:msg2} \\
z_{b}^{t+1} &\leftarrow (1-\gamma)z_{b}^{t} + 
\dfrac{1}{|N_b|}\sum_{a \in N_b} m^{t+1}_{a b} 
& &\mbox{for all $b\in \bV$},
\label{eq:msg3} \\
u_{ab}^{t+1} &\leftarrow u_{ab}^t + \gamma x_{ab}^{t+1} - z_{b}^{t+1} +
(1-\gamma) z_{b}^t  
& &\mbox{for all $(a,b) \in \bE$},
\label{eq:msg4} \\
n_{ab}^{t+1} &\leftarrow z_{b}^{t+1} - u_{ab}^{t+1} 
& &\mbox{for all $(a,b) \in \bE$} .
\label{eq:msg5}
\end{align}
\end{subequations}
Above, $\gamma \in (0, 2)$ is the over-relaxed parameter,
$\rho > 0$ is the penalty parameter, and $t$ is the iteration time.
One can check that \eqref{eq:admm_msg_updates} is consistent
with the standard non-distributed over-relaxed ADMM updates 
\cite{Boyd}.
We can see the above updates as a message passing algorithm as 
illustrated in Fig.~\ref{fig:alice}b. 
The only messages shared through the network are
$m_{ab}$ and $n_{ab}$, and every node keeps and updates a 
copy of the components of $\vu$  
corresponding to edges incident on itself. 
All the updates only require local information. 
This scheme is on the same
lines as the one proposed in \cite{Bento1,Bento2}.

\begin{figure}
\begin{minipage}{.49\textwidth}
\centering
\vspace{.5cm}
\includegraphics[scale=1]{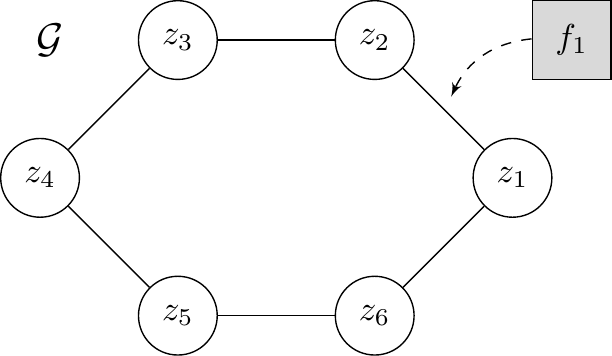}
\end{minipage}
\begin{minipage}{.49\textwidth}
\centering
\includegraphics[scale=1]{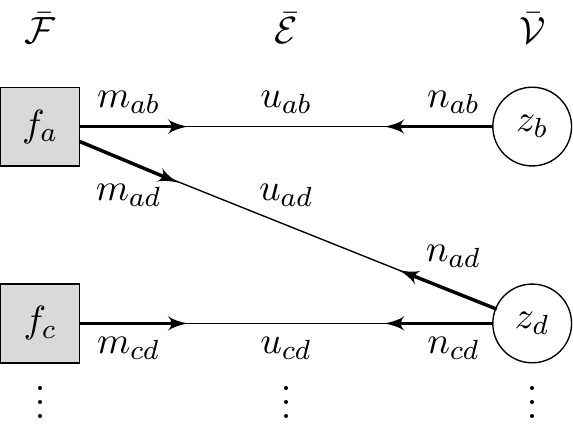}
\end{minipage}
\put(-122,-70){(b)}
\put(-360,-70){(a)}
\caption{
(a) An example of graph $\G$. The corresponding factor graph $\bG$ is obtained
by attaching a function node to every edge of $\G$.
(b) Distributed over-relaxed ADMM as a message passing algorithm over the 
factor graph. Messages $m_{ab}$ and $n_{ab}$ are shared
by agents solving local problems.
}
\label{fig:alice}
\end{figure}

Replacing the decoupled objective \eqref{eq:quadratic} 
explicitly into the updates \eqref{eq:admm_msg_updates},
and introducing the
variable $\vs = S \vz$, 
the above scheme can be 
written in the following matrix form:
%
\begin{align}\label{eq:admm_matrix}
\vx^{t+1} &=  A \vn^t , 
&\vu^{t+1} &= \vu^t + \gamma \vx^{t+1} + (1-\gamma) \vs^t - \vs^{t+1},  \\
\vm^{t+1} &= \gamma \vx^{t+1} + \vu^t ,
&\vs^{t+1} &= (1-\gamma) \vs^t + B \vm^{t+1} , \qquad \qquad
\quad \vn^{t+1} = \vs^{t+1} - \vu^{t+1} , \nonumber
\end{align}
%
where $S$ is defined in \eqref{eq:Sdef} and we have introduced the operators
\begin{equation}
\label{eq:AB}
A = \left(I + \rho^{-1} Q\right)^{-1},  \qquad 
B = S (S^\top S)^{-1} S^\top. 
\end{equation}
Note that $B=B^\top$ is symmetric and moreover
$B^2 = B$, thus it is 
an orthogonal  projection operator.
Its orthogonal complement is denoted by $B^\perp \equiv I - B$, 
satisfying
$BB^\perp = B^\perp B = 0$.

Although the updates \eqref{eq:admm_matrix} have a total of
$5 \times |\bE|$ dimensions, a result from
\cite{FrancaBentoMarkov} shows that these can be reduced to the
following linear
system in only $|\bE|$ dimensions:
\begin{equation}
\label{eq:TA}
\vn^{t+1} = T_A \vn^{t}, \qquad T_A = I - \gamma(A + B - 2BA),
\end{equation}
where all the other variables in \eqref{eq:admm_matrix} 
depend only on $\vn^t$.

We are interested in computing the convergence rate $\tau$ 
defined in \eqref{eq:conv_def}.
A straightforward adaptation of standard results from
Markov chain theory 
gives us the following.
\begin{theorem}[See \cite{MarkovChainsPeres}]\label{th:aux_the_mark_chain}
Consider the linear system 
$\bm{\xi}^{t+1} = T \bm{\xi}^{t}$, and let $\bm{\xi}^\star$ be a fixed point. 
If the spectral radius is $\rho(T)=1$
and it is attained by the eigenvalue $\lambda_1(T)=1$ with
multiplicity one,
then $\bm{\xi}^t = T^t\bm{\xi}^0$ converges to $\bm{\xi}^\star $ and satisfies
$\| \bm{\xi}^t - \bm{\xi}^\star\| = \Theta(|\lambda_2|^t)$, where $\lambda_2 =
\lambda_2(T)$ is the second largest eigenvalue of $T$ in absolute value
(the largest is $\lambda_1(T) = 1$).
\end{theorem}

Since $T_A$ in \eqref{eq:TA} is nonsymmetric, its eigenvalues can be
complex. We thus order them by magnitude:
\begin{equation}
|\lambda_1(T_A)| \ge |\lambda_2(T_A)| \ge \dotsm \ge |\lambda_{|\bE|}(T_A)|.
\end{equation}
When the order of a particular eigenvalue is not important, we drop the 
index and simply write $\lambda(T_A)$.

Notice  that the optimization problem \eqref{eq:gen_quad} 
is convex and has solution $\vz^\star = c \bm{1}$ for any constant
$c$, which spans a linear space of dimension one. 
It is straightforward to check
that $\lambda_1(T_A) = 1$ is unique (with eigenvector being the all-ones 
vector)
and every other
eigenvalue satisfies $|\lambda_n(T)| < 1$, for $n=2,\dotsc,|\bE|$. 
Due to Theorem \ref{th:aux_the_mark_chain}, the 
asymptotic convergence rate of ADMM is thus
determined by the second largest eigenvalue
$\tau_A =   |\lambda_2(T_A)|$.

\section{Computing the Spectrum of ADMM}
\label{sec:admm_spectrum}

As explained above, our problem boils down to finding the spectrum
of $T_A$. First, we write this operator in a more convenient form
(the proof can be found in Appendix~\ref{sec:lemma2}).
\begin{lemma}
\label{th:simplified_TA}
The matrix $T_A$ defined in \eqref{eq:TA} can be written as
\begin{equation} 
\label{eq:Ta_new}
T_A =  \left( 1 - \dfrac{ \gamma}{2}   \right) I 
+ \dfrac{\gamma}{\rho+2} U \qquad \mbox{where } \qquad
U = \OO + \dfrac{\rho}{2} \tB,
\end{equation}
with $\tB =\tB^\top = 2B - I$, $\OO = \tB \RR$, and $R = R^\top = I - Q$.
In particular, $\OO$ is orthogonal, i.e.
$\OO^\top \OO = \OO \, \OO^\top = I$,
and the other symmetric matrices satisfy $\tB^2 = I$ and $R^2 = I$.
\end{lemma}

Notice that the spectrum of $T_A$ can be easily determined once we know
the spectrum of $U$. In particular, 
if $\rho=0$ then $U=\Omega$ is orthogonal and its eigenvalues
lie on the unit circle in the complex plane. Thus, we 
may expect that for $\rho$
sufficiently small, the eigenvalues of $U$ lie in a perturbation of this
circle. It turns out that, in general, 
the eigenvalues of $U$ either lie on a circle in the complex plane, with
center at $1-\gamma/2$ and radius 
$\tfrac{\gamma}{2}\sqrt{(2-\rho)/(2+\rho)}$,
or on the real line. Furthermore, by exploring properties
of the matrices of Lemma~\ref{th:simplified_TA} 
we can calculate the spectrum of $U$ 
exactly for any $\rho > 0$ in terms of the spectrum
of the original graph $\G$. This is one of our main results, whose
proof is in Appendix~\ref{sec:ADMM_W}.

\begin{theorem}[ADMM and random walks on $\G$]
\label{th:omega_w}
Let $\W = \D^{-1} \AA$ be the probability transition matrix 
of a random walk on the graph
$\G$, where $\D$
is the degree matrix and $\AA$ the adjacency matrix.
For each eigenvalue $\lambda(\W) \in (-1,1)$, the matrix 
$T_A$ in \eqref{eq:Ta_new} has a pair of eigenvalues given by
\begin{equation}
\label{eq:EigenTA_random}
\lambda^{\pm}(T_A) = \left( 1 - \dfrac{\gamma}{2}\right)
+ \dfrac{\gamma}{2+\rho}\left(\lambda(\W) \pm i 
\sqrt{1-\dfrac{\rho^2}{4} - \lambda^2(\W)}\right).
\end{equation}
Conversely, 
any eigenvalue $\lambda(T_A)$ is of the 
form \eqref{eq:EigenTA_random} for some
$\lambda(\W)$.
\end{theorem}

In general, the eigenvalues of $T_A$ are complex.
However, $T_A$ always has the largest
eigenvalue $\lambda_1(T_A) = 1$, which can also be obtained from
\eqref{eq:EigenTA_random} if we  
replace $\lambda_1(\W) = 1$
and pick the negative sign in   \eqref{eq:EigenTA_random}.  
Another important real eigenvalue is the following (see
Appendix~\ref{sec:proof_tuning}).

\begin{lemma}
\label{th:one_gamma_eigen}
The matrix $T_A$ has eigenvalue $\lambda(T_A) = 1-\gamma$
if and only if the graph $\G$ has a cycle of even length.
\end{lemma}

In the results to follow we assume that
$T_A$ has the eigenvalue $1-\gamma$ since this encloses the most interesting
cases. We can still carry out the analysis when this is not the case, however
we omit these results for conciseness and simplicity.

Henceforth, we always assume that $\G$ has at least one cycle of
even length. 
Observe that, for many families of randomly generated graphs, 
this occurs with overwhelming probability. 
Consider sampling $\G$ from an Erd\" os-R\' enyi model with $n$ 
vertices and edge probability $p$.  
There are 
$C(n,k) = \left(\begin{smallmatrix} n \\ k \end{smallmatrix}\right)$ 
ways of choosing
$k$ vertices, and the probability that each set of $k$ nodes forms a cycle is at least $p^k$. Therefore,
the probability that there will be no $k$-cycle in $\G$ is upper bounded by
$(1-p^k)^{ C(n,k) }$ which is extremely small.

A few observations about $\W$ are in order.
It is known that the eigenvalues of $\W$ are
in the range $\lambda(\W) \in [-1,1]$. 
The second largest eigenvalue of 
$\W$, denoted by
$w^\star \equiv \lambda_2(\W)$, and the corresponding eigenvalue of $T_A$
from formula 
\eqref{eq:EigenTA_random},
play an important role in computing the optimal convergence rate $\tau^*_A$.
Moreover, the second largest eigenvalue $w^\star$ is related to the mixing time
of the Markov chain associated to $\W$, and also to the conductance
$\Phi \in [0,1]$ of the graph $\G$ by the Cheeger bound \cite{Cheeger}:
\begin{equation}
1-2\Phi \le
\omega^\star \le 1- \Phi^2/2. 
\end{equation}
The conductance $\Phi$
tells us whether or not $\G$ has bottlenecks,
and higher $\Phi$ implies a fast mixing chain 
\cite{lovasz1999faster}. 
The conductance of $\G$ is defined by
\begin{equation}\label{eq:conductance_for_random_walk}
\Phi =  \min_{\mathcal{S} \subset \V}
\dfrac{C(\mathcal{S})}{\sum_{i\in \V} d_i} \qquad \mbox{such that} \qquad
\sum_{i\in \V} d_i \leq |\V|
\end{equation}
where $d_i$ is the degree of node $i$ and $C(\mathcal{S})$ is cut-value
induced by $\mathcal{S}$, i.e. the number of edges that cross from
$\mathcal{S}$
to $\V \backslash \mathcal{S}$. 

In the context
of \cite{FrancaBentoMarkov}, which motivated this paper,
the most interesting cases
are Markov chains that have low conductance 
and are known to not 
speedup via lifting.
Therefore, we will present our results for graphs $\G$ where
the second largest eigenvalue of the transition matrix $\W$ lies in the range
$0 \leq \omega^\star < 1$, which is implied by $\Phi \leq 1/2$. 

We now discuss the behaviour of the eigenvalues of $T_A$.
From the formula \eqref{eq:EigenTA_random} we just need to analyse 
the complex eigenvalues of the operator $U$, defined
in \eqref{eq:Ta_new}, which are
given by $\lambda(U) = \lambda(\W) \pm i
\sqrt{1-\rho^2/4 - \lambda^2(\W)}$. 
Therefore, $\lambda(U)$
lies on a circle
of radius $\sqrt{1-\rho^2/4}$, and each eigenvalue becomes
real when $\rho^2/4+\lambda^2(\W) \ge 1$. All eigenvalues become real
when $\rho > 2$. When $\rho=0$ the circle has unit radius, and as $\rho$
increases the radius of the circle shrinks. Note that only the
imaginary part of $\lambda(U)$ changes with $\rho$, so every complex conjugate
pair of eigenvalues move vertically downwards until they fall on the 
real line, one moving to the left and
the other to the right. We illustrate this behaviour in 
Fig.~\ref{fig:eigenvalues} where we show the corresponding eigenvalues of $T_A$.
The eigenvalues marked in red move on the vertical dashed line
as we increase $\rho$.
Notice also from \eqref{eq:EigenTA_random} that 
as $\rho\to \infty$ all eigenvalues
tend to either $\lambda(T_A) \to 1$ or $\lambda(T_A) \to 1-\gamma$.

To tune ADMM we need to minimize the second largest, in 
absolute value, 
eigenvalue of $T_A$.
The minimum will come from either the conjugate pairs in
\eqref{eq:EigenTA_random} with $\omega^\star = \lambda_2(\W)$,
marked in red in Fig.~\ref{fig:eigenvalues}, or from
the real eigenvalue $1-\gamma$ of Lemma~\ref{th:one_gamma_eigen},
marked in green in Fig.~\ref{fig:eigenvalues}. We can keep increasing
$\rho$ to make the radius of the circle the smallest possible, which happens
when these complex eigenvalues have vanishing imaginary part. This determines
the best parameter $\rho^\star$. Now we can fix $\gamma^\star$ by making
$|1-\gamma|$ the same size as the norm of the previous complex conjugate 
eigenvalues.
Using these ideas we obtain our next result, whose proof is contained in
Appendix~\ref{sec:proof_tuning}.

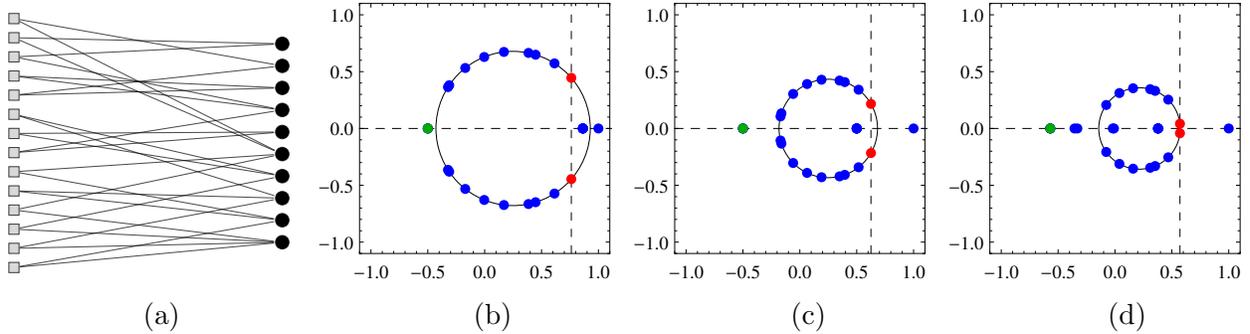
\begin{figure}
\includegraphics[width=\textwidth]{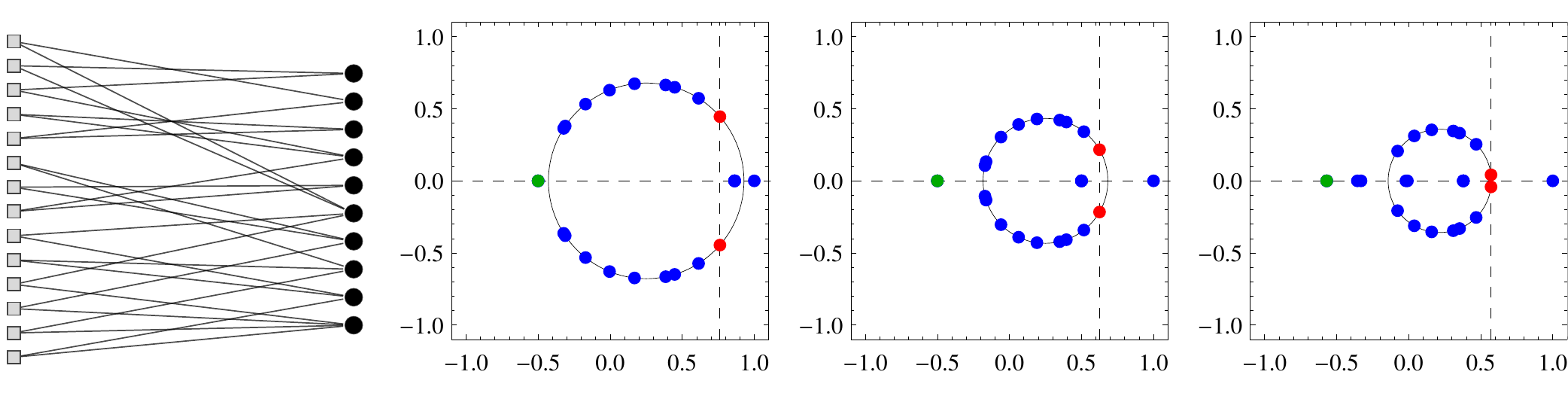}
\put(-415,-9){(a)}
\put(-290,-9){(b)}
\put(-170,-9){(c)}
\put(-50,-9){(d)}
\caption{
\label{fig:eigenvalues}
(a) The factor graph of a randomly generated
graph. The
squares are the functions $f_a \in \bF$, and the circles
are the variables
$z_b \in \bV$, where here $|\bF|=17$, $|\bV| = 10$, and $|\bE|=28$.
We numerically compute the eigenvalues of
\eqref{eq:Ta_new}. 
The complex eigenvalues lie on a circle of
radius 
$\tfrac{\gamma}{2}\sqrt{(2-\rho)/(2+\rho)}$, centered
at $1-\gamma/2$, as described
by the formula in Theorem~\ref{th:omega_w}. The red dots
correspond to the second largest, in magnitude, eigenvalues in 
\eqref{eq:EigenTA_random}. The green dot
is the eigenvalue $1-\gamma$ of
Lemma~\ref{th:one_gamma_eigen}. We have the following parameters:
(b) $\gamma = 1.5$, $\rho = 0.2$;
(c) $\gamma = 1.5$, $\rho = 1$;
(d) $\gamma = 1.59$, $\rho = 1.31$. This is close to the optimal
parameters of 
Theorem~\ref{th:ADMM_convergence_rate}, which are
$\gamma^\star = 1.56976$ and
$\rho^\star = 1.32181$, obtained with $\omega^\star = 0.75047$ which
is determined by the graph.
}
\end{figure}

\begin{theorem}[Optimal convergence rate for ADMM]
\label{th:ADMM_convergence_rate}
Assume that the graph $\G$ has at least one cycle of even length, and
conductance $\Phi \leq 1/2$.
Let $\W = \D^{-1} \AA$ be the transition matrix of a random walk on $\G$,
and denote its second largest
eigenvalue by $\omega^\star = \lambda_2(\W) \in (0,1)$.
Let $\lambda_2(T_A)$ be the second largest, in absolute value,
eigenvalue of $T_A$.
The best possible convergence rate of ADMM is thus given by
\begin{equation}
\label{eq:optimal_rate_ADMM}
\tau^\star_A \equiv \min_{\gamma, \rho} |\lambda_2(T_A)| = \gamma^\star - 1,
\end{equation}
where
\begin{equation}
\label{eq:optimal_gamma_rho_ADMM}
\gamma^\star = \dfrac{4}{3-\sqrt{(2-\rho^\star)/(2+\rho^\star)}} 
\qquad \mbox{ and } \qquad
\rho^\star = 2 \sqrt{1 - ({\omega^{\star}})^2}.
\end{equation}
\end{theorem}

The above theorem provides optimal parameter selection for over-relaxed
ADMM in terms of the second largest eigenvalue $\omega^\star = \lambda_2(\W)$ 
of the transition matrix, which captures the topology of $\G$. Recall
that $\omega^\star$ is also related to the well-known spectral gap.

\subsection{Graphs without even cycles}

We can still solve the analogous of Theorem~\ref{th:ADMM_convergence_rate}
when the graph $\G$ does not have
even length cycles, or $\G$ has high conductance. However, this does not
introduces new insights and slightly complicates the analysis. To be concrete,
we just state one of such cases below.

Consider a case where $\G$ does not have
a cycle of even length, for example when $\G$ is a tree, thus
$\lambda(T_A) = 1-\gamma$ does not exist.
The most interesting case is for slow 
mixing chains, $\Phi \leq 1/2$, and 
analogously to 
Theorem~\ref{th:ADMM_convergence_rate}
we obtain the following result.

\begin{theorem}
\label{th:ADMM_convergence_rate2}
Assume that the graph $\G$ has no cycles of even length, and has
conductance $\Phi \leq 1/2$.
Let $\W = \D^{-1} \AA$ be the transition matrix of a random walk on $\G$.
Denote the  second largest eigenvalue of the transition matrix $\W$ by 
$\omega^\star \in (0,1)$,
and $\bar{\omega}$ its smallest eigenvalue different than $-1$. Assume 
that $|\bomega| \geq \omega^*$. 
Let $\lambda_2(T_A)$ be the second largest, in absolute value,
eigenvalue of $T_A$.
The best possible convergence rate of ADMM is given by
\begin{equation}
\label{eq:optimal_rate_ADMM2}
\tau^\star_A \equiv \min_{\gamma, \rho} |\lambda_2(T_A)| = 
1 - \dfrac{\gamma^\star}{2}\left( 1 - \dfrac{2}{2+\rho^\star}\omega^\star
\right)
\end{equation}
where
\begin{equation}
\label{eq:optimal_gamma_rho_ADMM2}
\gamma^\star = 4 \left( 2- 
\dfrac{\omega^\star + \bomega - \sqrt{ \bomega^2 - (\omega^\star)^2
}}{1+\sqrt{1-(\omega^\star)^2}}
\right)^{-1}
\qquad \mbox{ and } \qquad
\rho^\star = 2 \sqrt{1 - ({\omega^{\star}})^2}.
\end{equation}
\end{theorem}

Notice that if we replace $\bomega=-1$ in the above formulas
we recover the results from 
Theorem~\ref{th:ADMM_convergence_rate}.
Furthermore, a straightforward calculation shows that the rate 
\eqref{eq:optimal_rate_ADMM} is always an upper bound for
the rate \eqref{eq:optimal_rate_ADMM2}. 
In fact, we can show that \eqref{eq:optimal_rate_ADMM} 
is always an upper bound for $\tau^*_A$
regardless of the topology of $\G$. 
We omit these results for simplicity, as well as a proof of 
Theorem~\ref{th:ADMM_convergence_rate2}
since it is analogous to the proof
of 
Theorem~\ref{th:ADMM_convergence_rate}.

\subsection{Numerical examples}

We provide some numerical experiments illustrating
our theoretical results by considering the graphs shown
in Table~\ref{tb:graphs}. The second largest eigenvalue of the transition
matrix, $\omega^\star$, is determined by the graph. The conductance $\Phi$
is computed by direct inspection of $\G$ and 
\eqref{eq:conductance_for_random_walk}. The other quantities are computed
from our theoretical predictions, e.g.
Theorem~\ref{th:ADMM_convergence_rate} and 
Theorem~\ref{th:ADMM_convergence_rate2}.
For each graph, in Fig.~\ref{fig:numerical} we show the corresponding
convergence rates from a numerical computation of the second largest eigenvalue of
$T_A$, denoted by $\lambda_2$. We fix several values of $\gamma$ and 
plot $|\lambda_2|$ versus $\rho$. The solid blue lines in the plots
correspond to our theoretical prediction for the optimal convergence 
rate $\tau_A^\star$, whose values are in Table~\ref{tb:graphs}. The red lines
show the convergence rate as function of $\rho$ for 
optimal $\gamma = \gamma^*$ from a numerical computation. 
Both curves touch under optimal parameter tuning, confirming our
theoretical predictions. The remaining curves show suboptimal rates.

For the graphs in (a) and (b) of Table~\ref{tb:graphs} the assumptions
of Theorem~\ref{th:ADMM_convergence_rate} 
hold, thus we can find optimal parameters through the formulas
\eqref{eq:optimal_rate_ADMM} and
\eqref{eq:optimal_gamma_rho_ADMM}, whose values are indicated.
In Fig.~\ref{fig:numerical}a and Fig.~\ref{fig:numerical}b
we can see that the optimal numerical rates (red lines) match
the prediction of formula \eqref{eq:optimal_rate_ADMM} (blue lines).
We also included two other curves
using the values $\gamma=1.3$ and $\gamma=1.6$ to show that the rates becomes
suboptimal if $(\rho,\gamma) \neq (\rho^*,\gamma^*)$.

For the graph in (c) the assumptions
of Theorem~\ref{th:ADMM_convergence_rate} do not hold since the conductance
$\Phi > 1/2$. A similar analysis as of
Theorem~\ref{th:ADMM_convergence_rate} 
shows that, for all graphs with even cycles and high conductance 
we have $\tau^*_A = 1/3$, $\gamma^* = 4/3$ and $\rho^* = 2$, which are the
values indicated in Table~\ref{tb:graphs}. We omit this proof for simplicity
of presentation. In Fig.~\ref{fig:numerical}c we show that a numerical
calculation matches this prediction (blue and red lines). The curves
with $\gamma_1=1.2$ and $\gamma_2=1.5$ give suboptimal rates.
A misapplication of 
Theorem~\ref{th:ADMM_convergence_rate}  gives $\rho_3 \approx 1.886$, 
$\gamma_3\approx1.414$ and $\tau_A^{(3)} \approx 0.414$, which still 
gives an upper bound on
the optimal $\tau_A^\star=1/3$. 
Using the value of $\gamma_3$ to numerically
compute $\lambda_2(\rho)$ yields the curve shown in dashed line. 

Theorem~\ref{th:ADMM_convergence_rate2} holds to the case of the graph
in item (d), whose predictions
are shown in Table~\ref{tb:graphs}. These agree with the numerical
results shown in Fig.~\ref{fig:numerical}d.
A misapplication 
of Theorem~\ref{th:ADMM_convergence_rate} 
yields $\rho_3\approx1.351$, $\gamma_3\approx1.563$ 
and $\tau_A^{(3)}\approx0.659$, which still
upper bounds $\tau_A^\star \approx 0.536$. Using $\gamma_3$ to compute
$\lambda_2(\rho)$ yields
the suboptimal rate shown in dashed line.

\begin{table}
\def\arraystretch{1.2}
\begin{tabular}{@{}m{.6cm}m{1cm}|m{.8cm}m{2.2cm}m{1.2cm}m{1.2cm}m{1.2cm}m{1.5cm}m{1.5cm}@{}}
\toprule[1pt]
& $\bm{\G}$ & $\bm{\Phi}$ & $\bm{\omega^\star}$ & $\bm{\rho^\star}$ & 
$\bm{\gamma^\star}$ & $\bm{\tau_A^\star}$ 
& \textbf{Thm.~\ref{th:ADMM_convergence_rate}} 
& \textbf{Thm.~\ref{th:ADMM_convergence_rate2}} 
\\
\midrule[0.5pt]
(a) & \includegraphics[width=.7cm,height=.7cm,trim=0 .6cm .6cm 0]{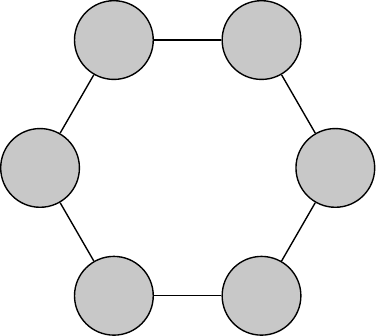} & 
$1/3$ & $1/2$ & $1.732$ & $1.464$ & $0.464$ & \cmark & \xmark \\
(b) & \includegraphics[width=.7cm,height=.7cm,trim=0 .6cm .6cm 0]{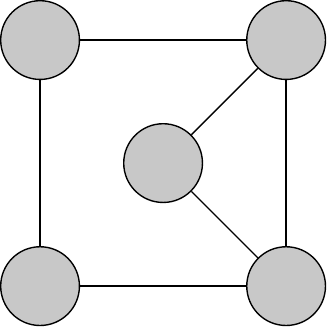} & 
$1/2$ & $1/3$ & $1.886$ & $1.414$ & $0.414$ & \cmark & \xmark \\
(c) & \includegraphics[width=.7cm,height=.7cm,trim=0 .6cm .6cm 0]{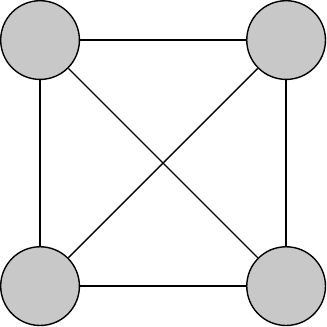}
&
$1$ & $-1/3$ & $2$ & $4/3$ & $1/3$ & \xmark & \xmark \\
(d) & \includegraphics[width=.7cm,height=.7cm,trim=0 .6cm .6cm 0]{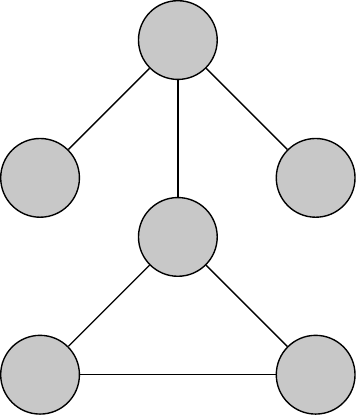}  &
$1/5$ & $\tfrac{1}{12}(\sqrt{97}-1)$ & $1.351$ & $1.659$ & $0.536$ & \xmark &
\cmark \\
\bottomrule[1pt]
\end{tabular}
\caption{\label{tb:graphs}
Application of Theorem~\ref{th:ADMM_convergence_rate} for the graphs 
in (a) and (b). For the graph in (c) the assumptions of 
Theorem~\ref{th:ADMM_convergence_rate} do not hold, but similar analysis
give the above results.
The graph in (d) has no even cycle, however 
Theorem~\ref{th:ADMM_convergence_rate2} applies 
with $\bar{\omega} = -\tfrac{1}{12}(\sqrt{97}+1)$, which is fixed by $\G$.
}
\end{table}
\begin{figure}
\includegraphics[width=\textwidth]{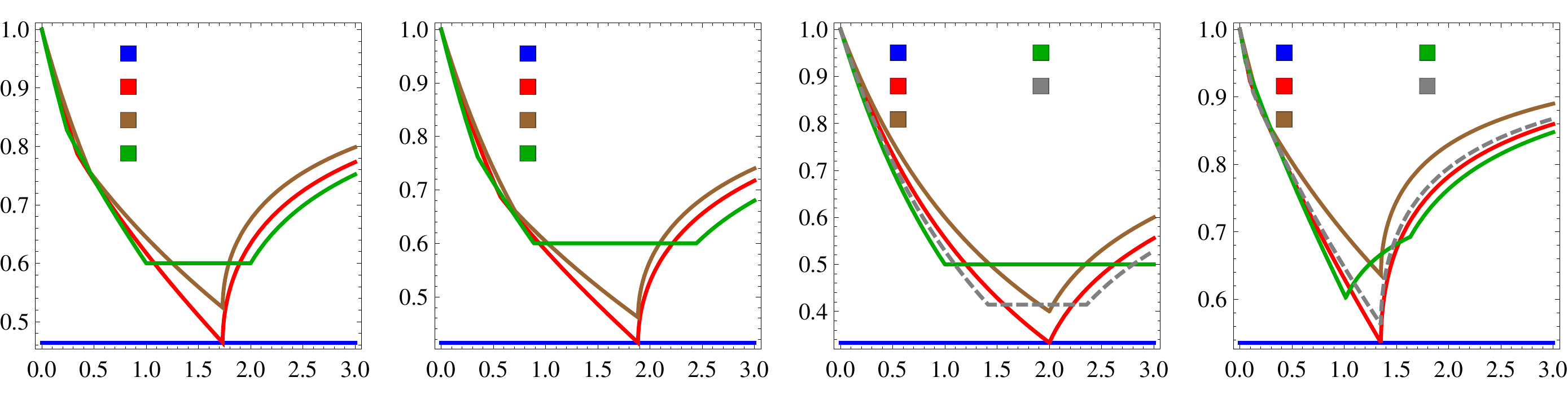}
\put(-415,-7){(a)}
\put(-295,-7){(b)}
\put(-175,-7){(c)}
\put(-58,-7){(d)}
\put(-425,101){\scriptsize{$\tau_A^\star$}}
\put(-425,91){\scriptsize{$|\lambda_2(\rho)|_{\gamma^\star}$}}
\put(-425,81){\scriptsize{$|\lambda_2(\rho)|_{\gamma_1}$}}
\put(-425,71){\scriptsize{$|\lambda_2(\rho)|_{\gamma_2}$}}
\put(-306,101){\scriptsize{$\tau_A^\star$}}
\put(-306,91){\scriptsize{$|\lambda_2(\rho)|_{\gamma^\star}$}}
\put(-306,81){\scriptsize{$|\lambda_2(\rho)|_{\gamma_1}$}}
\put(-306,71){\scriptsize{$|\lambda_2(\rho)|_{\gamma_2}$}}
\put(-195,101){\scriptsize{$\tau_A^\star$}}
\put(-195,91){\scriptsize{$|\lambda_2(\rho)|_{\gamma^\star}$}}
\put(-195,81){\scriptsize{$|\lambda_2(\rho)|_{\gamma_1}$}}
\put(-154,101){\scriptsize{$|\lambda_2(\rho)|_{\gamma_2}$}}
\put(-154,91){\scriptsize{$|\lambda_2(\rho)|_{\gamma_3}$}}
\put(-81,101){\scriptsize{$\tau_A^\star$}}
\put(-81,91){\scriptsize{$|\lambda_2(\rho)|_{\gamma^\star}$}}
\put(-81,81){\scriptsize{$|\lambda_2(\rho)|_{\gamma_1}$}}
\put(-37,101){\scriptsize{$|\lambda_2(\rho)|_{\gamma_2}$}}
\put(-37,91){\scriptsize{$|\lambda_2(\rho)|_{\gamma_3}$}}
\caption{
\label{fig:numerical}
Numerical results of
$\tau_A(\rho,\gamma) = |\lambda_2(\rho)|_{\gamma}$, where
$\lambda_2=\lambda_2(T_A)$ is the second largest eigenvalue of $T_A$, versus
$\rho$. We fix $\gamma$ for each curve, and
each figure corresponds to a graphs in Table~\ref{tb:graphs}.
Solid blue lines correspond to theoretical predictions for $\tau_A^\star$.
(a,b) $\gamma_1=1.3$ and $\gamma_2=1.6$.
(c) $\gamma_1 = 1.2$ and $\gamma_2 = 1.5$.
Formula \eqref{eq:optimal_gamma_rho_ADMM}
gives $\gamma_3 = 1.414$ and the suboptimal rate shown in dashed line. 
(d) 
$\gamma_1 = 1.3$ and $\gamma_2=1.8$.
A misapplication of 
\eqref{eq:optimal_rate_ADMM} gives $\gamma_3=1.659$, shown in dashed line.
}
\end{figure}
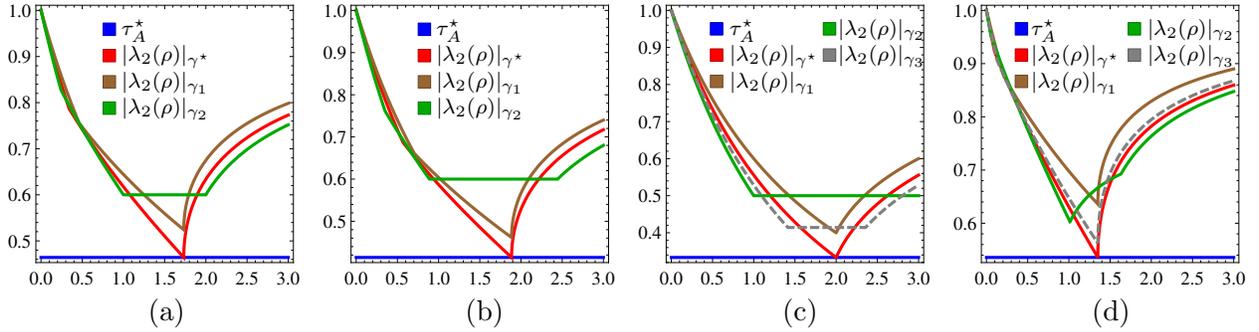

\section{Comparison with Gradient Descent}
\label{sec:admm_vs_gd}

We now compare the optimal convergence rates of distributed ADMM and GD,
denoted by $\tau_A^\star$ and $\tau_G^\star$, respectively. 
Let us first recall that the convergence rate of GD is related to eigenvalues
of the graph Laplacian.
We can write the objective function \eqref{eq:gen_quad} explicitly as 
\begin{equation}\label{eq:new_f_quad}
f(z) = \dfrac{1}{2} \sum_{i \in \V} \Big\{
d_i z_i^2  - 2 z_i \sum_{j \in N_i} z_j
+ \sum_{j \in N_i} z_j^2\Big\}
\end{equation}
where $N_i$ is the neighboring set of node $i \in \V$, and $d_i$ is its
degree.
Using the component form of GD update,
$z_k^{t+1} = z_k^t - \alpha\partial_{z_k}f(z^t)$, and noticing that
the last term of \eqref{eq:new_f_quad}
does not contribute since $i\ne j$, we obtain
%
$z_k^{t+1} = z_k^t - \alpha \big( 
d_k z_k^{t} - \sum_{j \in N_k} z_j^t \big)$.
%
Notice also that $\sum_{j\in N_k}z_j = \sum_{j \in \V}\AA_{kj}z_j$, 
where $\AA$ is the adjacency
matrix of $\G$. Therefore,  writing this result in matrix
form 
we have
\begin{equation}\label{eq:gd_laplacian}
\vz^{t+1} = T_G \vz^t, \qquad T_G = I - \alpha \L,
\end{equation}
where $\L \equiv \D - \AA$ is the Laplacian of $\G$.
Since the eigenvalues of $\L$ are real, we assume the following ordering:
$
\lambda_1(\L) \ge \lambda_2(\L) \ge \dotsc \ge 
\lambda_{|\E|-1}(\L) \ge \lambda_{|\E|}(\L)=0.
$

Let $\bell = \lambda_1(\L)$ be the largest eigenvalue of $\L$, 
and $\ell^\star = \lambda_{|\E|-1}$ be the second smallest and nonzero eigenvalue
of $\L$. We have
$\tau^\star_G = \min_{\alpha} \max 
\left\{ |1- \alpha \, \bell |  ,  |1 - \alpha \, \ell^\star| \right\} $
whose solution is
\begin{equation}
\label{eq:tau_g_star}
\tau_G^\star
= \dfrac{\bell - \ell^\star}{\bell + \ell^\star}.
\end{equation}
To relate this result with the transition matrix $\W$, 
note that
%
$\D^{1/2} \W \D^{-1/2} = \D^{-1/2} \AA \D^{-1/2} \equiv I - \LL$,
%
where $\LL$ is the 
\emph{normalized} Laplacian of $\G$. Thus,
both operators have the same eigenvalues,
%
$
\lambda(\W) =  1 - \lambda(\LL),
$
%
which are all real.
We now use the following
bounds \cite[Lemmas 2.12 and 2.21]{zumstein2005comparison}:
\begin{subequations}
\label{eq:bounds}
\begin{align}
d_{\textnormal{min}} \lambda_i(\LL) \leq  & \lambda_i(\L) \leq 
d_{\textnormal{max}} \lambda_i(\LL) \qquad (i = 1,\dotsc,|\E|),
\label{eq:bound1} \\
d_{\textnormal{max}} \leq & \lambda_{1}(\L) \leq 2 d_{\textnormal{max}},
\label{eq:bound2}
\end{align}
\end{subequations}
where $d_{\textnormal{max}}$ and $d_{\textnormal{min}}$ 
are the maximum and minimum degree 
of $\G$, respectively. 
Equation 
\eqref{eq:bound1} gives us
$ \ell^\star/d_{\max} \leq 1 - \omega^* \leq   \ell^\star/d_{\min}$, 
which together
with equation \eqref{eq:bound2} allows us to bound $\tau^*_G$ 
using $\omega^*$. Further
using an expansion of
\eqref{eq:optimal_rate_ADMM} around  $\omega^* = 1$ 
leads to the following result, whose proof can be found in
Appendix~\ref{sec:proof_conjecture}.

\begin{theorem}[ADMM speedup] 
\label{th:main_theorem_about_TA_and_TG}
Assume that the graph $\G$ has an even length cycle and $\Phi \leq 1/2$, such that 
Theorem~\ref{th:ADMM_convergence_rate} holds.
Then, there is
$C = 1 - \mathcal{O}\big(\sqrt{\delta}\big)$ such
that
\begin{equation}
\label{eq:ADMM_GD_Square}
C \big(1-\tau_G^\star\big) 
\le 
\big(1-\tau_A^\star\big)^2 \le
2 \Delta C \big(1-\tau^\star_G\big),
\end{equation}
where $\Delta = d_{\textnormal{max}}/d_{\textnormal{min}}$ is the ratio of
the maximum to the minimum degree of $\G$. 
Here $\delta = 1 - \omega^\star$ is the spectral gap.
\end{theorem}

The lower bound in \eqref{eq:ADMM_GD_Square}
provides a proof of conjecture \eqref{eq:conjecture}, proposed in
\cite{FrancaBentoMarkov}.
Notice that the upper bound in \eqref{eq:ADMM_GD_Square}
implies that
ADMM cannot improve much more than this square root factor. 
However, this upper bound 
becomes more loose for very irregular graphs, which have $\Delta \gg 1$,
compared
to regular graphs, which have $\Delta = 1$. 
Moreover, as briefly mentioned before, since 
Theorem~\ref{th:ADMM_convergence_rate} provides an upper bound on 
$\tau_A^\star$ regardless of the topology of $\G$,
the lower bound in 
\eqref{eq:ADMM_GD_Square} still remains valid for \emph{any} graph.
Numerical results illustrating the lower bound in
\eqref{eq:ADMM_GD_Square}
were already provided in \cite{FrancaBentoMarkov}.

\section{Final Remarks}
\label{sec:conclusion}

We provided a thorough analysis of 
distributed over-relaxed ADMM when solving
the non-strongly-convex consensus
problem \eqref{eq:gen_quad} in terms
of spectral properties of the underlying graph $\G$;
see Theorem~\ref{th:omega_w}.
The \emph{exact} asymptotic convergence rate of ADMM depends
on the second largest eigenvalue of the transition matrix of $\G$.
This result directly relates distributed ADMM to a Markov chain.
We also provided
explicit formulas for optimal parameter selection; see
Theorem~\ref{th:ADMM_convergence_rate} and
Theorem~\ref{th:ADMM_convergence_rate2}.
Comparing the optimal convergence
rates of distributed ADMM and GD, we were able to prove 
a recent conjecture based on a close analogy with 
lifted Markov chains~\cite{FrancaBentoMarkov}.
We showed that, for problem \eqref{eq:gen_quad} over 
\emph{any} graph $\G$, when both algorithms are
optimally tuned, distributed ADMM always provides
a speedup given by a square root factor compared to GD;
see Theorem~\ref{th:main_theorem_about_TA_and_TG}.

We believe that our results and methods may shed a new light into
distributed optimization, in particular for ADMM. For instance, it provides the
first steps towards a better understanding of how distributed ADMM behaves
when splitting a decomposable objective function.
It would be certainly desirable, and interesting,
to extend our analysis to more general
settings. We hope the results presented here motivate future
research in this direction.

\bigskip

\begin{acknowledgments}
We thank Peter Kravchuk for a brief discussion.
This work is supported by NIH/1U01AI124302 grant.
\end{acknowledgments}

\newpage
\appendix

\section{Proof of Lemma~\ref{th:simplified_TA}}
\label{sec:lemma2}

We repeat, then prove, 
the statement of Lemma~\ref{th:simplified_TA} in the main
text  below.
\begin{lemma}
\label{thA:simplified_TA}
The matrix $T_A$ defined in \eqref{eq:TA} can be written as
\begin{equation} 
\label{eqA:Ta_new}
T_A =  \left( 1 - \dfrac{ \gamma}{2}   \right) I 
+ \dfrac{\gamma}{\rho+2} U \qquad \mbox{where } \qquad
U = \OO + \dfrac{\rho}{2} \tB,
\end{equation}
with $\tB =\tB^\top = 2B - I$, $\OO = \tB \RR$, and $R = R^\top = I - Q$.
In particular, $\OO$ is orthogonal, i.e.
$\OO^\top \OO = \OO \, \OO^\top = I$,
and the other symmetric matrices satisfy $\tB^2 = I$ and $R^2 = I$.
\end{lemma}
\begin{proof}
Due to the block diagonal structure of $Q$, the matrix $A$ in \eqref{eq:AB}
can be written as
\begin{equation}
A =  I - \dfrac{1}{\rho + 2} Q.
\end{equation}
Write $Q = I - R$, where $R$ is block diagonal with each
block in the form
$R_a = \left( \begin{smallmatrix} 0 & 1 \\ 1 & 0  \end{smallmatrix}\right)$
for $a \in \bF$.
We therefore have 
\begin{equation}
A = \dfrac{\rho + 1}{\rho+2} \, I + \dfrac{1}{\rho + 2} \, R.
\end{equation}
Replacing this expression 
into  \eqref{eq:TA} it is possible to reorganize the terms
in the form \eqref{eq:Ta_new}.
\end{proof}

\section{Proof of Theorem~\ref{th:omega_w}}
\label{sec:ADMM_W}

We first present several intermediate results that will be necessary
to establish Theorem~\ref{th:omega_w} from the main text. 

\begin{lemma}
\label{th:useful}
Let $U$ be a nonsingular matrix, and let $V \equiv \tfrac{1}{2}
\left(U + \eta \, U^{-1}\right)$ for some constant $\eta$.
If $v$ is an eigenvalue of $V$, then
$U$ has at least one of the following eigenvalues:
\begin{equation}
\label{eq:eigen_upm}
u^{\pm}  = v \pm i \sqrt{\eta - v^2}.
\end{equation}
Conversely, every eigenvalue of $U$ has the form
\eqref{eq:eigen_upm} for either 
$u^+$, $u^-$ or both,
for some eigenvalue $v$ of $V$.
\end{lemma}
\begin{proof}
We have
$\det\left( V - v I\right) = 0$ if and only if $v$ is an eigenvalue of 
$V$. From the definition
of $V$ we can this write as
\begin{equation}
\det\left (\tfrac{1}{2} U^{-1} \right) 
\det\left( U - u^+ I \right)
\det\left( U - u^- I \right) = 0 .
\end{equation}
Since $\det U^{-1} \ne 0$ by assumption, at least one of the other
determinants must vanish, showing that either $u^+$ or $u^-$ (or both) are
eigenvalues of $U$.

For the second part, 
consider the eigenvalue equation 
$U \bm{u} = u \, \bm{u}$.
It follows that
$V \bm{u}  = (U + \eta U^{-1}) \bm{u} = \left(
\tfrac{1}{2}\left( u + \eta \, u^{-1} \right)\right) \bm{u}$, thus
for every eigenvalue $u$ we have that
$v = \tfrac{1}{2}\left( u + \eta \, u^{-1} \right)$ is an eigenvalue of $V$, 
or equivalently, 
$u$ satisfy the quadratic equation $u^2 - 2 v  u + \eta = 0$
for some eigenvalue $v$ of $V$. The roots of this equation are 
given by \eqref{eq:eigen_upm}, thus $u$ must be
equal to at least one of these roots. 
\end{proof}

\begin{lemma}
\label{th:inverse_of_U_and_relation_to_V}
Let
\begin{equation}
\label{eq:U1}
U = \Omega + \dfrac{\rho}{2} \widetilde{B},
\end{equation}
where $\Omega \equiv \widetilde{B} R$ is orthogonal, $\tB \equiv 2B - I$,
and the symmetric
operators 
$\widetilde{B}$ and $\widetilde{R}$ both satisfy 
$\widetilde{B}^2 = I$ and $R^2=I$.
The  inverse of $U$ is given by
\begin{equation}
U^{-1} = \bigg( 1-\dfrac{\rho^2}{4} \bigg)^{-1}
\bigg( \Omega^\top
- \dfrac{\rho}{2} \widetilde{B} \bigg).
\end{equation}
We also have the following relation for the symmetric part of $\Omega$:
\begin{equation}
\label{eq:V2} 
\Omega_S \equiv \dfrac{\Omega + \Omega^\top}{2}  = 
\dfrac{U + \eta \, U^{-1}}{2} \qquad \mbox{with}
\qquad \eta = 1 - \dfrac{\rho^2}{4}.
\end{equation}
\end{lemma}
\begin{proof}
This can be checked by direct substitution.
\end{proof}

\begin{lemma}
\label{th:range_omega_s}
The eigenvalues of $\Omega_S$ are in the range
$ [-1, 1]$.
\end{lemma}
\begin{proof}
This follows trivially from \eqref{eq:V2} 
and orthogonality of $\Omega$. 
The eigenvalues of $\Omega$
have the form $\lambda(\Omega) = e^{i \theta}$ for 
$\theta \in (-\pi, \pi]$. Since $\Omega^\top = \Omega^{-1}$, we have
$\lambda(\Omega_S) = \cos\theta \in
[-1, 1]$. 
\end{proof}

From  Lemma~\ref{th:useful} and Lemma~\ref{th:inverse_of_U_and_relation_to_V}
we immediately know that all eigenvalues
of \eqref{eq:U1} have the form \eqref{eq:eigen_upm} with 
$v \to \lambda(\Omega_S)$ and $\eta \to 1-\rho^2/4$,
for either $u^+$ or $u^-$.
Now if we exclude the extremes of the interval where $\lambda(\Omega_S)$ lie,
according to Lemma~\ref{th:range_omega_s},
we have a stronger version of this result.

\begin{corollary}
\label{th:pairs}
If $w_S \in (-1, 1)$ is an eigenvalue of $\Omega_S$, then
the operator \eqref{eq:U1}
has a pair of eigenvalues given by
\begin{equation}
\label{eq:eigen_u1}
u^\pm = w_S \pm i \sqrt{1-\dfrac{\rho^2}{4} - {(w_S)}^2}.
\end{equation}
\end{corollary}
\begin{proof}
Lemma~\ref{th:useful} already implies that $U$ have
eigenvalues \eqref{eq:eigen_u1} for at least one of the choices
$u^\pm$. It remains to show that both occur if $w_S \in  (-1,1)$.
First, consider the case where $\rho = 0$. We have $u^\pm = w_S \pm i
\sqrt{1-(w_S)^2}$, and since $|w_S| < 1$, both $u^{\pm}$ are a complex
conjugate pair. Since $U$ is real, its complex eigenvalues 
always occur in conjugate pairs, thus
both $u^\pm$ are eigenvalues of $U$. 

For small enough $\rho > 0$ the eigenvalues $u^\pm$ in \eqref{eq:eigen_u1}
are also complex, so both must be eigenvalues of $U$.
Therefore, for small enough $\rho$, the characteristic
polynomial of $U$ has a factor of the form
\begin{equation}
\label{eq:char_factor}
\det(U - u I) \sim (u - u^+)(u - u^-) = u^2 - 2 w_S u - (1-\rho^2/4).
\end{equation}
Now $\det(U - u I)$ is a polynomial in both $u$ and $\rho$, and  since
a polynomial is uniquely determined by its coefficients, the
same factors in \eqref{eq:char_factor} will be present in the characteristic
polynomial of $U$ for
any $\rho$, implying that
both $u^{\pm}$ are eigenvalues of $U$ for any $\rho > 0$.
\end{proof}

We will show that, 
if we restrict ourselves to the 
interval $(-1,1]$, the eigenvalues $w_S$ of $\Omega_S$
are the same as
the eigenvalues of the transition matrix $\W$ of the original graph 
$\G$. This establishes
the connection with the graph topology.
However, we first need several intermediate results.
We recall that $B$ and $R$ are defined by
\begin{equation}
\label{eq:BR}
B = S(S^\top S)^{-1} S^\top, \qquad 
R = \begin{pmatrix} 
\ddots &     & \\
       & R_a & \\
       &     & \ddots
\end{pmatrix},
\end{equation}
where $S$ is a row stochastic matrix defined in \eqref{eq:Sdef}, and
the blocks of $R$ 
have the form $R_a = \left(\begin{smallmatrix} 0 & 1 \\ 1 & 0
\end{smallmatrix}\right)$. Also, $S^\top S = \D$ is the degree matrix
of $\G$. Moreover, $B^2 = B$ and $R^2 = I$.

\begin{lemma}
\label{th:spec_BR_spec_W}
If $\omega$ is an eigenvalue of the transition matrix $\W$,
then  $\omega$ is also an eigenvalue of the operator $BR$.
Conversely, if $\omega \neq 0$ is an eigenvalue of $BR$,
then  $\omega$ is also an eigenvalue of $\W$.
\end{lemma}
\begin{proof}
The matrix $S$ defined in \eqref{eq:Sdef} has independent columns, 
therefore its left  pseudo-inverse is 
$S^+ = (S^\top S)^{-1}S^\top= \D^{-1} S^\top$, where $\D$ is the degree
matrix of $\G$. Note that
$B = SS^+$, and also that
the adjacency matrix of $\G$ is given by $\AA = S^\top R S$.
Hence
$\W \equiv \D^{-1} \AA = S^+ RS$, and we obtain the identity
\begin{equation}\label{eq:relation_W_BR}
BRS = S\W .
\end{equation}
Consider the eigenvalue equation 
$\W \bm{\omega} = \omega \,\bm{\omega}$, where $\bm{\omega} \neq \bm{0}$.
Acting with \eqref{eq:relation_W_BR} on $\bm{\omega}$ 
we have 
$BR(S\bm{\omega}) = S \W \bm{\omega} = \omega (S \bm{\omega})$. 
Since the columns of $S$
are independent we have that $S \bm{\omega} \neq \bm{0}$, therefore
$\omega$ is also an eigenvalue of $BR$. 

Consider the eigenvalue equation
$\bm{v}^\top (BR)  = b \, \bm{v}^\top$, 
where $b \ne 0$ and $\bm{v} \neq \bm{0}$.
Since $R=R^{-1}$ is invertible, $\vv^\top B 
= b \, \vv^\top R \neq \bm{0}$.
Now $B$ is a projection onto the columns of $S$, thus we also have
$\vv^\top S \neq \bm{0}^\top$. Multiplying \eqref{eq:relation_W_BR} 
by $\vv^\top$ on the left we conclude that
$\vv^\top BRS = b (\vv^\top S) = (\vv^\top S) \W$, and
$b$ is also an eigenvalue of $\W$.
\end{proof}

\begin{lemma}\label{th:spec_BR_spec_Omega_S}
We have that $\omega \notin \{-1,0,1\}$ is an 
eigenvalue of $BR$ if and only if it 
is an eigenvalue of $\Omega_S$.
\end{lemma}
\begin{proof}
We claim, and later prove, the following two facts:
\begin{enumerate}
\item \label{fact_one} $\omega \neq 0$ is an eigenvalue of $BRB$
if and only if it is an eigenvalue of $BR$. 
\item \label{fact_two} $\omega \notin \{-1,0,1\}$ is an eigenvalue
of $BRB$ if and only if it is an eigenvalue of $-B^\perp R B^\perp$.
\end{enumerate}

We first prove that if 
$\omega$ is an eigenvalue of  $\Omega_S$, 
then $\omega$ is also an eigenvalue of $BR$.
From \eqref{eq:V2}, and recalling that $\tB = 2B - I$ and
$B + B^\perp = I$, 
we can write 
\begin{equation}
\label{eq:newOmegaS}
\begin{split}
\Omega_S &= RB - B^\perp R  \\
& = (B+B^\perp)RB - B^\perp R(B+B^\perp)  \\
& = BRB - B^\perp R B^\perp,
\end{split}
\end{equation}
where $B$ and $B^\perp$ are projectors onto orthogonal 
subspaces. From 
\eqref{eq:newOmegaS}
and using  the identity $\vv = B \vv + B^\perp \vv$, the eigenvalue
equation 
$\Omega_S \vv = \omega \vv$ (where $\vv \neq {\bf 0}$) is equivalent to 
\begin{equation}
\label{eq:brbperp}
BR (B \vv) = \omega (B \vv), \qquad
-B^\perp R (B^\perp \vv) = \omega (B^\perp \vv).
\end{equation}
Since $\vv \neq 0$, either $B  \vv \neq {\bf 0}$ or
$B^\perp  \vv \neq {\bf 0}$ (or both). Thus, if 
$\omega$ is an eigenvalue of 
$\Omega_S$, then $\omega$ is an eigenvalue of $BRB$ or 
an eigenvalue of $-B^\perp R B^\perp$ (or both). 
Assuming $\omega \notin \{-1,0,1\}$, by the fact \ref{fact_two} above
the operators $BRB$ and $-B^\perp R B^\perp$
have the same eigenvalues. Therefore, if 
$\omega$ is an eigenvalue of  $\Omega_S$, then it is also an eigenvalue 
of $BRB$, and by fact \ref{fact_one} it is also an eigenvalue of $BR$.

Now we prove the reverse.
If $\omega' \neq 0$ is an eigenvalue of $BR$, then
by fact \ref{fact_one} it is also an eigenvalue of $BRB$, i.e. 
$BRB  \vv' = \omega' \vv'$  
for some $\vv' \neq {\bf 0}$.
Acting on this equality with $B$ on both sides
we conclude that $B \vv' = \vv'$. 
Hence, using \eqref{eq:newOmegaS} and $B^\perp \vv' = {\bf 0}$ we obtain
$\Omega_S \vv' = BRB  \vv' = \omega' \vv'$, i.e.
$\omega'$ is an eigenvalue of $\Omega_S$.

The above two paragraphs proves the claim, now we finally
finally show that the above two facts hold.

\paragraph*{Proof of Fact \ref{fact_one}.} Let $\omega \neq 0$ be such that
$BRB \vv = \omega \vv$ for some $\vv \neq  0$.
Dividing this expression by $\omega$ we conclude 
that $\vv$ is in the range of $B$. 
Since $B$ is an orthogonal projection, $B\vv = \vv$.
The same argument holds if
$\omega \neq 0$ is an eigenvalue of $BR$.
Therefore, $BRB \vv = BR \vv= \omega \vv$, as claimed.

\paragraph*{Proof of Fact \ref{fact_two}.} 
We first argue that if $\omega \notin \{-1,0,1\}$ 
is an eigenvalue of $BRB$, then $\omega$ is
an eigenvalue of $B^\perp R B^\perp$. The argument for the other 
direction is the same with $B$ and $B^\perp$ switched.
Let $BRB \vv = \omega \vv$ for some $\vv \neq {\bf 0}$.
Since $\omega \neq 0$, we have that $B \vv = \vv$. 
Let $\vu  \equiv B^\perp  R \vv$.
We show that $\vu$ is an eigenvector of $-B^\perp R B^\perp $ 
with eigenvalue $\omega$.
We have $B^\perp R B^\perp  \vu  = B^\perp R B^\perp R \vv = 
B^\perp R (I - B) R \vv = B^\perp \vv - B^\perp R B R \vv = 
- B^\perp R B R \vv  = - B^\perp R B R B\vv = - \omega (B^\perp R \vv)
= -\omega \vu$. 
In addition, $\vu = B^\perp R \vv = R \vv - BR\vv = R \vv - BRB\vv = 
(R  -\omega I)\vv$. The eigenvalues of $R$ are $\pm1$, 
thus $(R  -\omega I)$ is non singular and it 
follows that $\vu \neq 0$.
\end{proof}

\begin{lemma}\label{th:singular_Omega_S_singular_W}
The transition matrix $\W$ is singular if and only if the operator 
$\Omega_S$ is singular.
\end{lemma}
\begin{proof}
From the proof of Lemma~\ref{th:spec_BR_spec_W} we know that
$\W = S^+ R S$. 
Suppose $\W$ is singular, i.e. there is 
$\vu \neq \bm{0}$ such that $\W \vu = S^+ R S \vu=\bm{0}$.
Since the columns of $S$ are independent and $R$ 
is invertible,
$\vv \equiv R S \vu \neq \bm{0}$, and therefore $S^+  \vv  = \bm{0}$.
Using \eqref{eq:V2} we can write $\Omega_S = BR - R + RB$,
and noticing that 
$R^2 = I$ and $B S = SS^+S = S$, we obtain
$\Omega_S  \vv = RB \vv= RS(S^+\vv) = \bm{0}$,
implying that $\Omega_S$ is also singular.

Suppose $\Omega_S$ is singular, i.e. there is $\vv \neq \bm{0}$ 
such that $\Omega_S \vv  = \bm{0}$.
From \eqref{eq:newOmegaS}
and noticing that $B$ and $B^\perp$ project onto orthogonal subspaces 
we have  
\begin{equation}
\label{eq:orth_cond_for_null_eign_of_Omega_S}
BRB \vv  = \bm{0} , \qquad B^\perp R B^\perp \vv  = \bm{0}.
\end{equation}
Consider two separate  cases. 
First, if $B \vv = \bm{0}$ then $B^\perp \vv = \vv \neq \bm{0}$,
and from 
equation \eqref{eq:orth_cond_for_null_eign_of_Omega_S} we have
$B^\perp R B^\perp \vv = B^\perp R \vv  = R\vv - BR \vv= \bm{0}$,
or $R \vv = S \vu$ where $\vu = S^+ R \vv \ne \bm{0}$.
Therefore, $\W \vu = S^+ R S \vu = S^+ \vv = \bm{0}$ showing
that $\W$ is singular, 
where we have 
used $R^2 = I$ and $B \vv = \bm{0}$ if and only if
$S^+\vv = \bm{0}$.
Second, suppose $B \vv \neq \bm{0}$, and let
$\vu  \equiv S^+ \vv \ne \bm{0}$. 
From the first equation in \eqref{eq:orth_cond_for_null_eign_of_Omega_S} 
we have $S \W \vu = SS^+ R SS^+ \vv = B R B \vv = \bm{0}$, 
but since $S$ has independent columns we must have
$\W \vu = \bm{0}$, which shows that $\W$ is singular.
\end{proof}

\begin{lemma}
\label{th:omegas_w}
We have that $\omega \in (-1,1]$ is an eigenvalue of the transition matrix
$\W$ if and only if it is
an eigenvalue of the symmetric operator $\Omega_S$. 
\end{lemma}
\begin{proof}
Combining Lemma \ref{th:spec_BR_spec_W} 
and Lemma \ref{th:spec_BR_spec_Omega_S} it follows that  
$\omega \notin \{-1,0,1\}$ is an eigenvalue of
$\W$ if and only if it is an eigenvalue of $\Omega_S$.
By Lemma \ref{th:singular_Omega_S_singular_W} we can 
extend this to $\omega = 0$.
Finally, $\W$ and $\Omega_S$ always have an eigenvalue $\omega=1$
with eigenvector being the all-ones vector.
\end{proof}

Finally, we are ready to show one of our main results, which
relates the spectrum of ADMM to the spectrum of random walks on $\G$.
We first repeat the statement of Theorem~\ref{th:omega_w} in the main text
for convenience.

\begin{theorem}[ADMM and random walks on $\G$]
\label{thA:omega_w}
Let $\W = \D^{-1} \AA$ be the probability transition matrix
of a random walk on
the graph $\G$, where $\D$
is the degree matrix and $\AA$ the adjacency matrix.
For each eigenvalue $\lambda(\W) \in (-1,1)$ the matrix
$T_A$ in \eqref{eq:Ta_new} has a pair of eigenvalues given by
\begin{equation}
\label{eqA:EigenTA_random}
\lambda^{\pm}(T_A) = \left( 1 - \dfrac{\gamma}{2}\right)
+ \dfrac{\gamma}{2+\rho}\left(\lambda(\W) \pm i 
\sqrt{1-\dfrac{\rho^2}{4} - \lambda^2(\W)}\right).
\end{equation}
Conversely,
any eigenvalue $\lambda(T_A)$ is of the
form \eqref{eqA:EigenTA_random} for some
$\lambda(\W)$.
\end{theorem}
\begin{proof}
The first part is an immediate consequence of 
Corollary~\ref{th:pairs} and Lemma~\ref{th:omegas_w}.
The second part is a consequence of Lemma~\ref{th:useful}.
\end{proof}

\section{Proof of Theorem~\ref{th:ADMM_convergence_rate}}
\label{sec:proof_tuning}

To establish Theorem~\ref{th:ADMM_convergence_rate}, which involves
graphs with even length cycles and low conductance, several intermediate
results will be needed.

\begin{lemma}
\label{th:remark_about_simul_diag} 
The operator $U$, defined in \eqref{eq:U1}, and also the symmetric operator
$\Omega_S$ are diagonalizable.
Moreover, $U$ and $\Omega_S$ commute and have a common eigenbasis.
\end{lemma}
\begin{proof}
It is obvious that $\Omega_S = \frac{\Omega + \Omega^\top}{2}$ 
is diagonalizable since
it is symmetric and real.
Let 
$U = P J P^{-1}$ be a decomposition in terms of a Jordan 
canonical form $J = \diag(J_1, J_2,\dotsc)$, where $J_i$ is the Jordan
block associated to eigenvalue $\lambda_i$.
From Lemma~\ref{th:inverse_of_U_and_relation_to_V}
we have
$\Omega_S = \tfrac{1}{2}(U + \eta U^{-1}) = 
\tfrac{1}{2}P(J + \eta J^{-1})P^{-1}$.
For every Jordan block $J_i$  
there is a corresponding Jordan block of the 
same dimension in $J + \eta J^{-1}$ with corresponding 
eigenvalue $\lambda_i  + \eta \lambda_i^{-1}$. 
Therefore, we can decompose $J + \eta J^{-1} = F Z F^{-1}$, 
where $Z$ is in Jordan form and has the same set of 
Jordan-block-dimensions as $J$, 
except that the diagonal values are different.
To mention an example, consider
\begin{equation}
J =  \left[
\begin{array}{cccccc}
 \lambda_1  & 0 & 0 & 0 & 0 & 0 \\
 0 & \lambda_2  & 1 & 0 & 0 & 0 \\
 0 & 0 & \lambda_2  & 0 & 0 & 0 \\
 0 & 0 & 0 & \lambda_3  & 1 & 0 \\
 0 & 0 & 0 & 0 & \lambda_3  & 1 \\
 0 & 0 & 0 & 0 & 0 & \lambda_3  \\
\end{array}\right],
\quad Z =  
\left[
\begin{array}{cccccc}
 \mu_1 & 0 & 0 & 0 & 0 & 0 \\
 0 & \mu_2 & 1 & 0 & 0 & 0 \\
 0 & 0 & \mu_2 & 0 & 0 & 0 \\
 0 & 0 & 0 & \mu_3 & 1 & 0 \\
 0 & 0 & 0 & 0 & \mu_3 & 1 \\
 0 & 0 & 0 & 0 & 0 & \mu_3 \\
\end{array}
\right], \quad \mu_i = \lambda_i + \eta \lambda_i^{-1} .
\nonumber
\end{equation}
Thus, we can write $\Omega_S =\tfrac{1}{2} (HF) Z (HF)^{-1}$.
The Jordan form of a matrix is unique, and $\Omega_S$ is diagonalizable, 
therefore,
all blocks in $Z$ must have dimension $1$, and so
does $J$, which means that $U$ is diagonalizable.

It is obvious that $U$ and $\Omega_S$ commute
due to \eqref{eq:V2}. Two diagonalizable
matrices that commute can be simultaneous diagonalizable, thus they
share a common eigenbasis.
\end{proof}

\begin{lemma}
\label{th:form_of_eig_of_omega_s}
If $w_S \in \{-1,1\}$ is an eigenvalue of $\Omega_S$ with corresponding
eigenvector $\vv$, then
\begin{itemize}
\item if $B\vv \neq \bm{0}$, 
then $B\vv$ is also an eigenvector of $\Omega_S$ and of
$R$ with eigenvalue $w_S$, i.e.
$\Omega_S (B\vv) = w_S (B\vv)$ 
and $R (B\vv) = w_S (B\vv)$. 
\item if $B^\perp\vv \neq \bm{0}$, then $B^\perp \vv$ is also an eigenvector
of $\Omega_S$ with eigenvalue $w_S$ and of $R$ with eigenvalue $-w_S$, i.e.
$\Omega_S (B^\perp \vv) = w_S (B^\perp \vv)$ 
and $R (B^\perp \vv) = -w_S (B^\perp \vv)$. 
\end{itemize}
\end{lemma}
\begin{proof}
Let $\Omega_S \vv = w_S \vv$ 
where $w_S \in\{-1,1\}$ and $\vv \neq {\bf 0}$. From
\eqref{eq:newOmegaS}
we have
\begin{equation}
\label{eq:BRBeqs}
BR(B\vv) = w_S (B\vv), \qquad B^\perp R (B^\perp \vv) = - w_S (B^\perp \vv).
\end{equation}
Assuming $B \vv \ne \bm{0}$ we have
$\Omega_S (B \vv) = BR(B\vv) = w_S (B\vv)$, which shows that $B\vv$ is
an eigenvector of $\Omega_S$ with eigenvalue $w_S$. Taking the norm on each
side of this equation and using $|w_S| = 1$ implies
\begin{equation}
\label{eq:BRB_norm}
\| BRB \vv \| = \| B \vv \|.
\end{equation}
Since $B$ is a projection operator, 
if $RB{\vv}$ is not in the span of $B$ 
then we must have 
$\|B(RB{\vv})\| < \|RB{\vv}\| \leq \|B{\vv}\|$, where the last 
inequality follows by using $\|R\| \leq 1$. However, this contradicts
\eqref{eq:BRB_norm}. Therefore,  $RB\vv$ must be in the span of $B$ and
as a consequence 
$R(B{\vv}) = BRB{\vv} = w_S (B{\vv})$, where we used the first
equation in \eqref{eq:BRBeqs}. This shows that
$B{\vv}$ is an eigenvector of $R$ with eigenvalue $w_S$ and completes the
proof of the first claim.

The proof of the second claim
is analogous. Assuming $B^\perp \vv \ne \bm{0}$ we obtain
$\Omega_S(B^\perp \vv) = -B^\perp R B^\perp \vv = w_S (B^\perp \vv)$, where
in the last passage we used the second equation in \eqref{eq:BRBeqs}.
This shows that $B^\perp \vv$ is an eigenvector of $\Omega_S$ with
eigenvalue $w_S$.
Taking the norm of this last equality 
yields 
\begin{equation}
\label{eq:BperpRBperp_norm}
\| B^\perp R B^\perp \vv\| = \| B^\perp \vv\|.
\end{equation}
Assuming that $RB^\perp \vv$ is not in the span of $B^\perp$
we conclude that
$\| B^\perp RB^\perp \vv\| < \| B^\perp \vv\|$, which contradicts
\eqref{eq:BperpRBperp_norm}. 
Therefore, we must have $B^\perp R B^\perp \vv = 
R(B^\perp \vv) = -w_S (B^\perp \vv)$, where we used \eqref{eq:BRBeqs}. 
This shows 
that $B^\perp \vv$ is an eigenvector of $R$ with eigenvalue $-w_S$.
\end{proof}

\begin{lemma}\label{th:eigen_prop_implies_graphs_prop}
If $B \vv \neq \bm{0}$ is an eigenvector of $R$ with eigenvalue
$-1$, then the graph $\G$ does not have odd-length cycles.
If $B^\perp \vv \neq \bm{0}$ is an eigenvector of $R$ with eigenvalue
$1$, then the graph $\G$ has cycles.
\end{lemma}
\begin{proof}
Define $\vx \equiv B \vv$ and $\vy \equiv B^\perp \vv$. We
index the components of $\vx,\vy \in \mathbb{R}^{|\bE|}$ 
by the edges of the factor graph. For instance,  $x_e$ and $y_e$ refers
to 
the respective component of $\vx$ and $\vy$ over the edge $e\in \bE$.
We look at
$B,R \in \mathbb{R}^{|\bE|\times|\bE|}$ as operators on edge
values.
Recall that $R$ in \eqref{eq:BR} has blocks in the form
$R_a = \left( \begin{smallmatrix} 0 & 1 \\ 1 & 0\end{smallmatrix} \right)$
for $a \in \bF$, thus each block has 
eigenvalues $\pm 1$ with eigenvectors $(1,\pm1)^\top$,
respectively.
Recall also that $B$ in \eqref{eq:BR} 
replaces the value of a given edge $e=(a,b)$ by the
average of the values over all the edges $(c,b)$ incident on 
$b\in \bV$, i.e. $(B\vv)_e = \tfrac{1}{|N_b|}\sum_{e'\sim b} v_{e}$, 
where $e' \sim b$ denotes that $e'$ is incident on node $b$.

Let us consider the first statement.
From the eigenvalues and eigenvectors of $R$, 
for every pair of edges $e_i,e_j \in \bE$ incident on a given 
function node $a \in \bF$ we have
\begin{equation}
\label{eq:ca}
x_{e_i} = -x_{e_j} = c_a \qquad \mbox{for $e_i,e_j \sim a\in\bF$},
\end{equation}
for some constant $c_a$.  
Now $B \vx = \vx \neq {\bf 0}$, 
thus for every set of edges $\{e_1,e_2,\dotsc,e_k\}$ 
incident on a variable node $b \in \bV$ we
have 
\begin{equation}
\label{eq:cb}
x_{e_1}=x_{e_2}=\dotsm=x_{e_k}= c_b,
\end{equation}
where $c_b$ is a constant.
Since the graph $\G$, and consequently its corresponding factor graph $\bG$,
is connected, 
we must have 
\begin{equation}
\label{eq:cacb}
|c_a|= |c_b| \qquad \mbox{for all $(a,b)\in \bE$}.
\end{equation}
Now assume that $\G$ has an odd cycle. This cycle must
traverses an odd number of 
variable nodes $a \in \bF$, but each pair of edges incident on $a \in \bF$ 
must have the same absolute value and opposite signs due to \eqref{eq:ca}, 
while each pair of edges incident on
a variable node $b \in \bV$ must have equal signs due to \eqref{eq:cb}. 
This implies that
$c_a = c_b$ and $c_a = -c_b$ for some $(a,b) \in \bE$, whose solution
is $c_a = c_b = 0$. From \eqref{eq:cb} this implies that for all
$e_i \sim b$ we have $x_{e_i}=0$, which in turn by \eqref{eq:cacb}
implies that $c_a=0$ for all $a$ incident upon these edges $e_i$, 
$i=1,\dotsc,k$, and so
forth. This yields
$\vx = \bm{0}$, which contradicts our assumption. Therefore,
$\G$ cannot have odd-length cycles. See Fig.~\ref{fig:ring_chain}a for
an example.

\begin{figure}
\begin{minipage}{.49\textwidth}
\vspace{2em}
\includegraphics{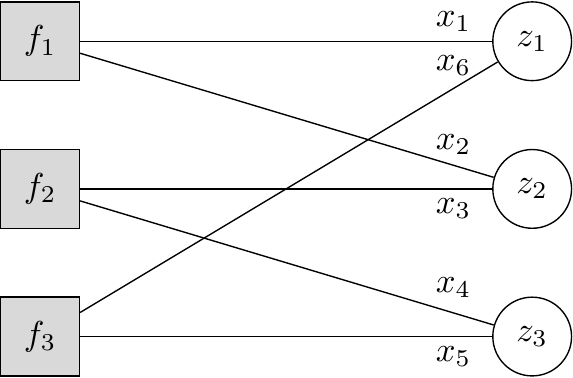}\\[.9em]
(a)
\end{minipage}
\begin{minipage}{.49\textwidth}
\includegraphics{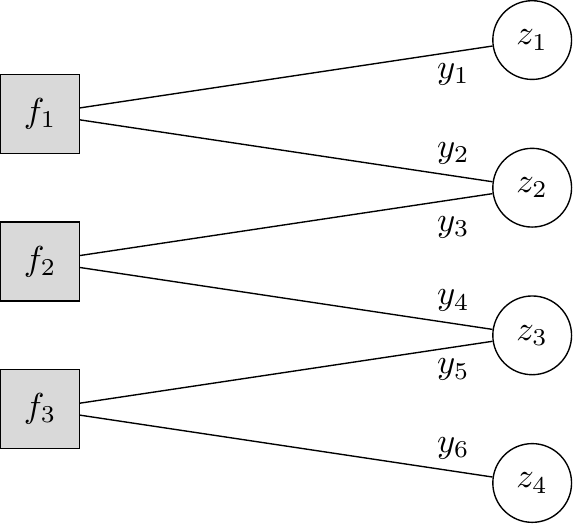}\\[-1em]
(b)
\end{minipage}
\caption{
\label{fig:ring_chain}
(a) $\bG$ with an odd-length cycle. Condition \eqref{eq:ca} 
requires $x_1=-x_2$, $x_3=-x_4$, $x_5=-x_6$, while \eqref{eq:cb}
requires
$x_1=x_6$, $x_2=x_3$, $x_4=x_5$. The only possible solution is
$x_i=0$ for $i=1,\dotsc,6$.
(b) $\bG$ having no cycle. Condition \eqref{eq:yij} requires
$y_1=y_2$, $y_3=y_4$, $y_5=y_6$, while \eqref{eq:yplus} requires
$y_1=0$, $y_2+y_3=0$, $y_4+y_5=0$, $y_6=0$. The only possible
solution is $y_i=0$ for $i=1,\dotsc,6$.
}
\end{figure}

Now we consider the second statement.
Assume that $\G$ has no cycles, since $\G$ and $\bG$ are connected, 
both must be trees.
Notice that $R \vy = \vy \neq {\bf 0}$ implies 
that for every pair of edges $e_i,e_j$
incident on $a \in \bF$ we have 
\begin{equation}
\label{eq:yij}
y_{e_i} = y_{e_j}
\end{equation}
On the other hand $B \vy = \bm{0}$, which requires that for every
set of edges $\{e_1, e_2,\dotsc,e_k\}$ incident on $b \in \bV$ we have
\begin{equation}
\label{eq:yplus}
y_{e_1} + y_{e_2} + \dotsm + y_{e_k} = 0
\end{equation}
The tree $\bG$ must have leaf nodes which are 
variable nodes because all function nodes have
degree $2$ and thus cannot be leaves. Consider a leaf node
$b \in \bV$ which must have only one incident 
edge $e_i=(a,b)$ for some $a\in \bF$.
Due to \eqref{eq:yplus}, we must have $y_{e_i} = 0$.
Denote the other edge incident on $a\in\bF$ by $e_j = (a,c)$ for
some $c\in \bV$. By \eqref{eq:yij} we 
also have $y_{e_j} = 0$. This implies
that the components of $\vy$ incident on $c \in \bV$ will also vanish. 
Since the graph is connected, and
propagating this argument for all nodes of the graph, we get
$\vy = \bm{0}$, which contradicts the assumption. Therefore,
$\G$ must have a cycle. See Fig.~\ref{fig:ring_chain}b for an example.
\end{proof}

\begin{lemma}\label{th:eigen_value_pert}
Let $u(\rho)$ be an eigenvalue of a matrix 
$U(\rho)$, depending on parameter
$\rho \in \mathbb{R}$, and such that $U(0) = \Omega$ where
$\Omega$ is orthogonal.
If $u(\rho)$ and $U(\rho)$ are differentiable at $\rho = 0$,
then
\begin{equation}
\label{eq:perturb}
\dfrac{{\rm d} u(0)}{{\rm d} \rho} = 
\bm{\omega}^\dagger \dfrac{{\rm d} U(0)}{{\rm d} \rho} \, \bm{\omega} 
\end{equation}
for some normalized eigenvector $\bm{\omega}$ of $\Omega$ 
with corresponding eigenvalue $u(0)$. Here
$\bm{\omega}^\dagger$ denotes the conjugate transpose of $\bm{\omega}$.
\end{lemma}
\begin{proof}
Since by assumption $u(\rho)$ and $U(\rho)$ are differentiable 
at $\rho =0$, they are well defined in a neighborhood of $\rho = 0$,
and therefore the following right- and left-eigenvalue
equations hold in such a neighborhood:
\begin{equation}
\label{eq:eigen_rho}
U(\rho) \bm{x}(\rho) = u(\rho) \bm{x}(\rho)
\end{equation}
where $\bm{x}(\rho)$ is some normalized eigenvector, i.e.
$\bm{x}(\rho)^\dagger \bm{x}(\rho) = 1$, and
\begin{equation}
\label{eq:eigen_rho_left}
\bm{y}(\rho)^\dagger U(\rho) = u(\rho) \bm{y}(\rho)^\dagger,
\end{equation}
where $\bm{y}(\rho)$ is some normalized eigenvector, 
i.e. $\bm{y}(\rho)^\dagger \bm{y}(\rho) = 1$.
Note that for each $\rho$ these equations might hold
for infinitely many $\bm{y}(\rho)$ and $\bm{x}(\rho)$.
We do not commit to any particular choice yet, but later we will
make a specific choice for certain values of $\rho$.

Define 
$\delta u(\rho) = u(\rho) - u(0)$,
$\delta U(\rho) = U(\rho) - U(0)$, and 
$\delta \bm{x}(\rho) = \bm{x}(\rho)-\bm{x}(0)$. 
From \eqref{eq:eigen_rho} we have 
\begin{equation}
\left[ U(0) + \delta U(\rho)\right]\bm{x}(\rho) = \left[ u(0)+\delta
u(\rho) \right]\bm{x}(\rho).
\end{equation}
Multiplying this equation on the left by $\bm{y}(0)^\dagger$ and using
\eqref{eq:eigen_rho_left} we obtain
\begin{equation}
\label{eq:gen_perb_eig_form} 
\delta u (\rho) \bm{y}(0)^\dagger \bm{x}(\rho)= 
\bm{y}(0)^\dagger \delta U(\rho) \bm{x}(\rho).
\end{equation}

Let $\{\rho_k\}$ be a sequence that converges to $0$.
For each $\rho_k$ fix a vector for the corresponding $\bm{x}(\rho_k)$
out of the potential infinitely many that might
satisfy \eqref{eq:eigen_rho}.
From $\bm{x}(\rho_k)^\dagger \bm{x}(\rho_k) = 1$ we know
that $\{ \bm{x}(\rho_k) \}$ is bounded. Therefore, 
there is a subsequence $\{\rho_{k_i}\}$ such 
that $\{\bm{x}(\rho_{k_i})\}$ converges to some limit vector 
$\bm{x}(0)$, which is also normalized. 
At this point we define $\bm{w} = \bm{x}(0)$.
From \eqref{eq:eigen_rho} and the continuity of $u(\rho)$
and $U(\rho)$ at $\rho=0$, we know that $\bm{w}$ 
satisfies $U(0)\bm{w} = u(0) \bm{w}$, i.e. $\bm{w}$ is an 
eigenvector of $U(0)=\Omega$ with eigenvalue $u(0)$. 
Since $U(0)$ is orthonormal, its left and
right eigenvectors are equal. 
Therefore, we also choose $\bm{y}(0) = \bm{x}(0) = \bm{w}$.

Dividing 
\eqref{eq:gen_perb_eig_form} 
by $\rho_{k_i}$
we can write 
\begin{equation} 
\dfrac{\delta u(\rho_{k_i})}{\rho_{k_i}} 
\bm{y}(0)^\dagger \bm{x}(\rho_{k_i}) =
\bm{y}(0)^\dagger 
\dfrac{\delta U(\rho_{k_i})}{\rho_{k_i}} 
\bm{x}(\rho_{k_i}).
\end{equation}
Taking the limit as $i \rightarrow \infty$ and using
differentiability of $u$ and $U$ at the origin, and also
the fact that $\bm{x}(\rho_{k_i}) \rightarrow \bm{x}(0) = \bm{w}$, we 
finally obtain \eqref{eq:perturb}.
\end{proof}

\begin{lemma}\label{th:no_odd_cycles_implies_eigen_values_not_present}
If the graph $\G$ does not have even length cycles, either $\Omega_S$, defined
in \eqref{eq:V2}, does not
have eigenvalue $-1$, or $- 1 - \rho/2$ is not an eigenvalue of $U$, 
defined in \eqref{eq:U1}.
\end{lemma}
\begin{proof}
From Lemmas~\ref{th:useful}~and~\ref{th:inverse_of_U_and_relation_to_V}
we know that all eigenvalues of $U$ must have the form \eqref{eq:eigen_u1}
for one of the sign choices and some eigenvalue $w_S$ of $\Omega_S$.
Since, by Lemma~\ref{th:range_omega_s}, we have $w_S \in [-1,1]$, 
the only way to obtain 
the eigenvalue $-1-\rho/2$ from
\eqref{eq:eigen_u1} 
is with $w_S = -1$ and a plus sign, 
which we denote by $u^+(-1) = -1-\rho/2$.
From Lemma~\ref{th:remark_about_simul_diag} we also 
know that $\Omega_S$ and $U$ are
both diagonalizable and commute, therefore, using a common eigenbasis,
any eigenvector $\vv$ of $U$ with eigenvalue $u^+(-1)$ must
also be an eigenvector of $\Omega_S$ with eigenvalue $w_S=-1$.

Henceforth, assume that $\G$ does not have even length cycles. Moreover,
assume that the following two eigenvalue equations hold:
\begin{equation}
\label{eq:noteigen2}
\Omega_S \vv = -\vv, \qquad
U\vv = -(1+\rho/2) \vv \qquad (\vv \neq \bm{ 0}) ,
\end{equation}
where $\vv$ is any normalized common eigenvector of $U$ and $\Omega_S$,
with respective eigenvalues $-(1+\rho/2)$ and $-1$,
and it does not depend on $\rho$.
We will show that these
assumptions lead to a contradiction, which proves the claim.

If \eqref{eq:noteigen2} holds, then 
$\tfrac{d}{d\rho}u^+(-1) = -\tfrac{1}{2}$, and by 
Lemma~\ref{th:eigen_value_pert}
we must have $\bm{\omega}^\dagger \tB \bm{\omega} = -1$ for some normalized 
eigenvector $\bm{\omega}$ of $U(0) = \Omega$ with eigenvalue $-1$. 
Now
\eqref{eq:noteigen2} is valid for $\rho = 0$, i.e. $U(0) \vv = - \vv$ for
any eigenvector $\vv$ with eigenvalue $-1$, therefore 
it is also valid for the vector $\bm{\omega}$. Thus, 
let us choose $\vv = \bm{\omega}$.
Using $\tB \equiv 2B - I = B - B^\perp$
we have 
\begin{equation}
\label{eq:vbv}
\vv^\dagger(B - B^\perp) \vv = -1.
\end{equation}
Note that $\| \vv^\dagger B \vv\| \le \| \vv \|^2 \| B\| = 1$, since
$\| \vv \|=1$ and $\| B \|=1$. 
Here $\|\cdot\|$ is the Euclidean norm for complex vectors. 
Moreover, $B$ is a symmetric (and thus Hermitian)
positive semidefinite matrix, which
means that $\vz^\dagger B \vz$ is real and non-negative 
for any non-zero complex
vector $\vz$. We thus have 
$\vv^\dagger B \vv \in [0,1]$, and  
analogously $\vv^\dagger B^\perp \vv \in
[0,1]$. From these facts and \eqref{eq:vbv} we conclude that
$\vv^\dagger B^\perp \vv = 1 $, 
which is equivalent to $B^\perp \vv = \vv \ne
\bm{0}$.
Furthermore, this immediately gives $B \vv = \bm{0}$.
Now, from the second item in 
Lemma~\ref{th:form_of_eig_of_omega_s} we have that 
$R (B^\perp \vv) = (B^\perp \vv)$, and upon using
Lemma~\ref{th:eigen_prop_implies_graphs_prop} we conclude that
$\G$ must have cycles.

\begin{figure}
\includegraphics{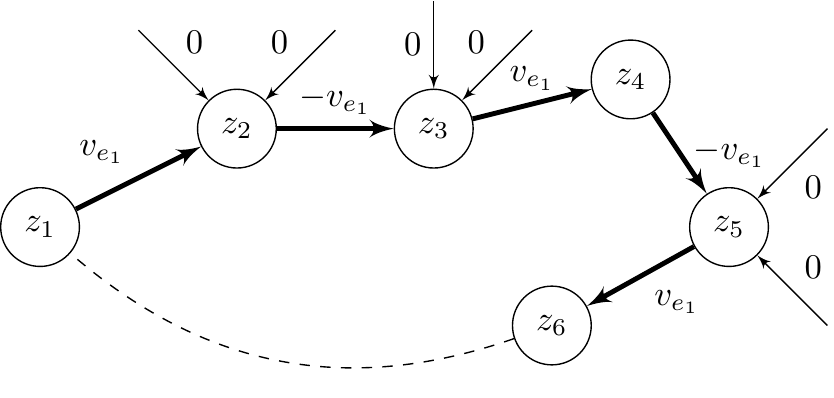}
\caption{
\label{fig:path_even}
The path $\mathcal{P}$ must be a cycle of even length otherwise
$\vv = \bm{0}$.
}
\end{figure}

To summarize, by assuming \eqref{eq:noteigen2} we concluded
that for $\vv \ne \bm{0}$ we have
\begin{equation}
\label{eq:concl} 
R\vv = \vv, \qquad 
B\vv = \bm{0},
\end{equation}
and that $\G$, and thus also $\bG$, has cycles.
The first eigenvalue equation in \eqref{eq:concl} requires
that pairs of edges incident in every function node $a \in \bF$ obey
\begin{equation}
\label{eq:concl3}
v_{e_i} = v_{e_j} \qquad 
\mbox{for $e_i,e_j \sim a\in\bF$},
\end{equation}
while the second equation in \eqref{eq:concl} requires
that all edges incident on variable nodes $b\in \bV$ add up to zero,
\begin{equation}
\label{eq:concl4}
\sum_{e\sim b} v_e = 0 \qquad \mbox{for all $b\sim \bV$}.
\end{equation}
We now construct a path $\mathcal{P}\subseteq\G$
while
obeying equations \eqref{eq:concl}. This obviously induces a path
$\bar{\mathcal{P}}$ on the associated factor graph $\bar{\G}$.
The
edges of $\G$ assume the values of the components of $\vv$, and 
we require that all  
edges in $\mathcal{P}$ are nonzero. 
First, note that \eqref{eq:concl3} imply that incoming and outgoing edges
of a function node must have the same value, thus if one edge is
nonzero it assures that the other edge is also nonzero. This
means that  $\bar{\mathcal{P}}$
cannot end on a function node. Therefore, we can remove 
function nodes 
altogether
from the picture and just think about edges and
variable nodes from the base graph $\G$. In this case, the only difference
compared to $\bG$ is that every edge in $\G$ will be duplicated in $\bG$.
Let us construct $\mathcal{P} \subseteq \G$ demanding that it has only nonzero
and non-repeating edges, 
and when we encounter a node which has an incident nonzero edge
we must move through this node. Furthermore, all the other edges which are
not part of $\mathcal{P}$ are set to zero.
Since $\vv\ne \bm{0}$ there exists at least one component
$v_{e_1} \ne 0$ over some edge $e_1=(z_1, z_2)$.
We start on $z_1 \in \bV$ and move to $z_2\in\bV$.
Because of \eqref{eq:concl4} the node $z_1$ cannot be a leaf
node, since this would require $v_{e_1} = 0$. Therefore, there exist
another edge $e_2=(z_2,z_3)$ with value $v_{e_2} = -v_{e_1}\ne 0$.
We thus move from $z_2$ to $z_3$, 
which again requires that over $e_3=(z_3,z_4)$ we have
$v_{e_3} = -v_{e_2} \ne 0$, and so on. See Fig.~\ref{fig:path_even} for
an illustration. Following this procedure,  
every edge in $\mathcal{P}$ has a nonzero
value, thus  $\mathcal{P}$ cannot end on any node, which implies that
it must be a cycle. Since all the edges in $\mathcal{P}$ have the
same value but alternating signs, $v_{e_1} = -v_{e_2} = v_{e_3} = -v_{e_4} =
\dotsm$, there must be
an even number of nodes in $\mathcal{P}$, otherwise we would have
$\vv = \bm{0}$. Therefore, we conclude that $\mathcal{P}$ must be
an even length cycle, which contradicts our original assumption.
This means that if $\G$ does not have even length cycles, both equations
\eqref{eq:noteigen2} cannot simultaneously hold.
\end{proof}

\begin{lemma}
\label{th:Uum}
If the graph $\G$ has an even length cycle, then the operator $\Omega_S$ 
has eigenvalue $-1$, and correspondingly $u^+(-1)=-1-\rho/2$ 
is an eigenvalue of $U$.
\end{lemma}
\begin{proof}
In what follows we index the entries of $\vv \in \mathbb{R}^{|\bE|}$ 
by the edges in $\bE$, thus $v_e$ is the 
value of edge $e\in \bE$.
We look at $B,R \in \mathbb{R}^{|\bE|\times|\bE|}$ as operators on 
edge values.
We will explicitly  construct an eigenvector $\vv\in \mathbb{R}^{|\bE|}$ 
such that $\Omega_S \vv = -\vv$ and $U \vv = -(1+\rho/2)\vv$.

If $\G$ has an even length cycle, then $\bG$ has a cycle,
which we denote by $\bar{\mathcal{C}}$ and it must 
cross an even number of function nodes $a\in\bF$.
For every $a \in \bF$ which is also part of the cycle $\bar{\mathcal{C}}$, 
let $e_i = (a,\,\cdot\,)$ and $e_j = (a,\,\cdot\,)$ be the two different
edges incident on $a$, and let $v_{e_i} = v_{e_j} = 1$.
For every $b \in \bV$ which is also part of $\bar{\mathcal{C}}$, pick
two different edges incident on $b$ and let
$v_{e_k} = -v_{e_\ell} = 1$, where $e_k = (\,\cdot\,, b)$ 
and $e_\ell=(\,\cdot\,, b)$.
For the remaining edges $e \in \bE$ which are not part of the cycle
$\bar{\mathcal{C}}$, let $v_e = 0$. With these requirements we satisfy
$R\vv = \vv$, 
since each function node has incident
edges of equal values $+1$ or $0$,
and $B \vv = \bm{0}$,  since each variable node have pairs of
incident edges with opposite signs $\pm 1$ or $0$. 
See Fig.~\ref{fig:even_length} for an example.

\begin{figure}
\includegraphics{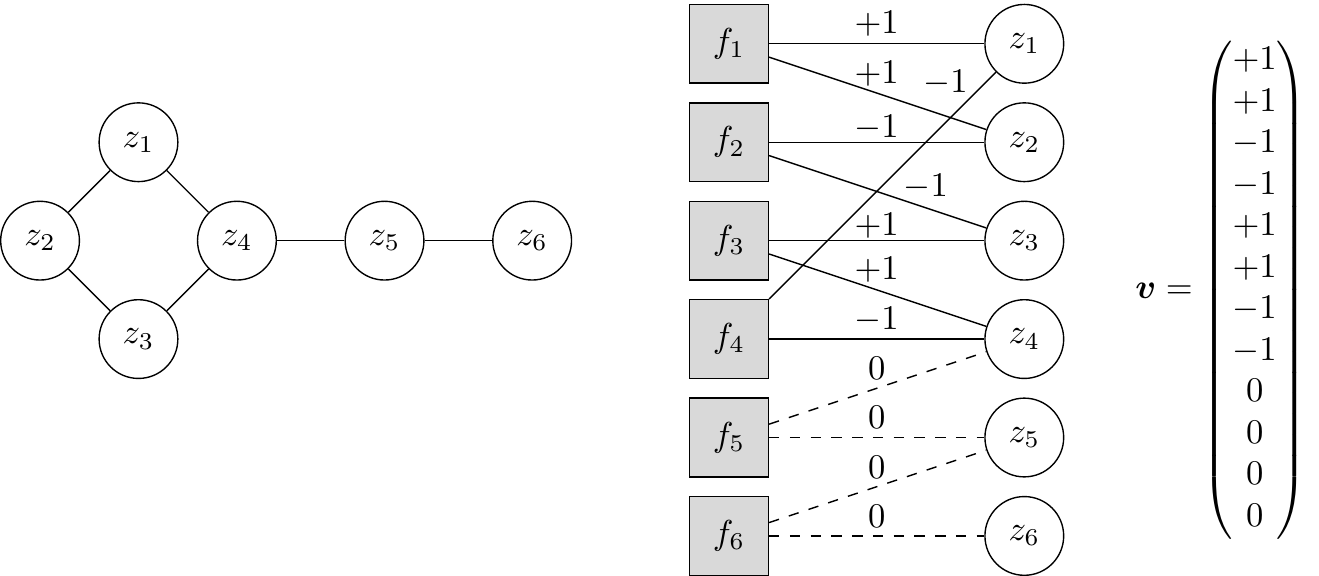}
\caption{
\label{fig:even_length}
Example of $\G$ having even length cycle, where 
$R\vv = \vv$ and $B\vv = \bm{0}$. The solid lines on the factor graph
$\bG$ indicate edges on the cycle $\mathcal{C}$, and the dashed lines edges
not in $\mathcal{C}$
}
\end{figure}

We explicitly constructed $\vv$ such that $R \vv = \vv$ and $B\vv = \bm{0}$.
This last equation immediately implies that $B^\perp \vv = \vv$.
From \eqref{eq:newOmegaS} we have
$\Omega_S = BRB - B^\perp R B^\perp$, therefore $\Omega_S \vv = -\vv$.
From \eqref{eq:U1} we have $U = (2B - I)R + \tfrac{\rho}{2}(2B - I)$, hence
$U \vv = -(1+\rho/2) \vv$, as claimed.
\end{proof}

We are now ready to prove Lemma~\ref{th:one_gamma_eigen} from the main
text, which is restated for convenience.

\begin{lemma}
\label{thA:one_gamma_eigen}
The matrix $T_A$ has eigenvalue $\lambda(T_A) = 1-\gamma$
if and only if the graph $\G$ has a cycle of even length.
\end{lemma}
\begin{proof}
We know from Lemma \ref{thA:simplified_TA} that 
$T_A$ has eigenvalue $1-\gamma$ if and only
if $U$ has eigenvalue $-1 - \frac{\rho}{2}$.
In addition, from Lemma~\ref{th:useful},
for $U$ to have eigenvalue $-1 - \frac{\rho}{2}$ 
it must be that $\Omega_S$ has eigenvalue $-1$.
With this in mind the remainder of the proof 
follows directly from
Lemma~\ref{th:no_odd_cycles_implies_eigen_values_not_present} and 
Lemma~\ref{th:Uum}.
\end{proof}

Finally, we show another important result from the main text,
Theorem~\ref{th:ADMM_convergence_rate}, which provides optimal parameter
tuning for ADMM when the graph $\G$ has even length cycles and low conductance.
The formulas for the parameters depend explicitly on the second largest
eigenvalue of the transition matrix $\W$.
We first restate the theorem for convenience.

\begin{theorem}[Optimal convergence rate for ADMM]
\label{thA:ADMM_convergence_rate}
Assume that the graph $\G$ has at least one cycle of even length, and 
conductance $\Phi \leq 1/2$.
Let $\W = \D^{-1} \AA$ be the transition matrix of a random walk on $\G$,
and denote its second largest
eigenvalue by $\omega^\star = \lambda_2(\W) \in (0,1)$.
Let $\lambda_2(T_A)$ be the second largest, in absolute value,
eigenvalue of $T_A$.
The best possible convergence rate of ADMM is thus given by
\begin{equation}
\label{eqA:optimal_rate_ADMM}
\tau^\star_A \equiv \min_{\gamma, \rho} |\lambda_2(T_A)| = \gamma^\star - 1,
\end{equation}
where
\begin{equation}
\label{eqA:optimal_gamma_rho_ADMM}
\gamma^\star = \dfrac{4}{3-\sqrt{(2-\rho^\star)/(2+\rho^\star)}} 
\qquad \mbox{ and } \qquad
\rho^\star = 2 \sqrt{1 - ({\omega^{\star}})^2}.
\end{equation}
\end{theorem}
\begin{proof}
First we need to determine the second largest eigenvalue of $T_A$ in absolute
value, denoted by $\lambda_2(T_A)$. From Theorem~\ref{thA:omega_w}
all the complex eigenvalues are centered at $1-\gamma/2$.
The real eigenvalue $\lambda(T_A) = 1 -\gamma$ 
is a distance $\gamma/2$ apart from the center, and so does 
$\lambda_1(T_A)$, and we know these are points on the extremes of the
interval where all real eigenvalues can lie. Since we are
not interested in $\lambda_1(T_A)=1$, the eigenvalue
$\lambda(T_A) = 1-\gamma$ can potentially be the
second largest since $0< \gamma < 2$. However, it
does not depend on $\rho$ so we can control its magnitude
by choosing $\gamma$ appropriately. 

Thus let us focus on the remaining
eigenvalues.
Every real eigenvalue of $T_A$ is at a smaller distance
than $\gamma/2$ from the center.
The second largest real eigenvalue of $T_A$ is obtained
from $u^-(1)$
which is at a distance
\begin{equation}
\label{eq:second_real_distance}
\dfrac{\gamma}{2}\left(\dfrac{2-\rho}{2+\rho}\right)
\end{equation}
from the center of the circle.
On the other hand, recall that any complex eigenvalue
is at a distance
\begin{equation}
\label{eq:distance_complex}
\dfrac{\gamma}{2}\sqrt{\dfrac{2-\rho}{2+\rho}}
\end{equation}
from the center, which is larger than \eqref{eq:second_real_distance}.
Therefore, besides $\lambda(T_A) = 1 - \gamma$ that can be controlled, 
$\lambda_2(T_A)$ must come from a complex conjugate pair
for some $0 < \lambda(\W) < 1$ in \eqref{eqA:EigenTA_random}. We
have
\begin{equation}
\label{eq:eigen_abs}
|\lambda^{\pm}(T_A)|^2 = \left(1-\dfrac{\gamma}{2}\right)^2
+ 2\left( 1-\dfrac{\gamma}{2} \right)\dfrac{\gamma}{2+\rho} \lambda(\W)
+ \dfrac{\gamma^2}{4}\dfrac{2-\rho}{2+\rho}.
\end{equation}
The first and third terms 
in \eqref{eq:eigen_abs}
do not depend on $\lambda(\W)$ and are the same for any eigenvalue.
Thus,
we must choose  the second largest eigenvalue
$\omega^\star = \lambda_2(\W)$, 
since we already excluded $\lambda_1(\W)=1$.
Thus
\begin{equation}\label{eq:lambda_2_gen}
\lambda_2(T_A) = \left(1-\dfrac{\gamma}{2}\right) +
\dfrac{\gamma}{2+\rho}\left( \omega^\star \pm i
\sqrt{1-\rho^2/4-(\omega^\star)^2} \right)
\qquad (\rho < 2).
\end{equation}

Notice that $\lambda_2(T_A)$ has smallest absolute value
when its imaginary part vanishes. Thus we can set
\begin{equation}\label{eq:optimal_rho}
\rho^\star = 2 \sqrt{1-(\omega^\star)^2}
\end{equation}
which gives
\begin{equation}\label{eq:optimal_TA}
\lambda^\star_2(T_A)=1-\dfrac{\gamma}{2} + 
\dfrac{\gamma \omega^\star}{2+\rho}
\end{equation}
Now we can make the remaining eigenvalue $|\lambda(T_A)| = |1-\gamma| $
match
\eqref{eq:optimal_TA}.
Writing $\omega^\star = \tfrac{1}{2}\sqrt{(2-\rho^\star)(2+\rho^\star)}$ 
and solving for $\gamma$ yields
\begin{equation}\label{eq:optimal_gamma}
\gamma^\star = \dfrac{4}{3+\sqrt{\tfrac{2-\rho^\star}{2+\rho^\star}}}.
\end{equation}
Finally,
$\tau_A^\star = \min |\lambda_2(T_A)| = |1-\gamma^\star|$ with parameters
given by \eqref{eq:optimal_rho} and \eqref{eq:optimal_gamma}.
\end{proof}

\section{Proof of Theorem~\ref{th:main_theorem_about_TA_and_TG}}
\label{sec:proof_conjecture}

Our last result is Theorem~\ref{th:main_theorem_about_TA_and_TG} from
the main text which proves conjecture \eqref{eq:conjecture}, 
proposed based 
on an analogy with lifted Markov chains \cite{FrancaBentoMarkov}.
Let us first restate the theorem.

\begin{theorem}[ADMM speedup]
\label{thA:main_theorem_about_TA_and_TG}
Assume that the graph $\G$ has an even length cycle and conductance
$\Phi \leq 1/2$, such that
Theorem~\ref{thA:ADMM_convergence_rate} holds.
Then, there is
$C = 1 - \mathcal{O}\big(\sqrt{\delta}\big)$ such
that
\begin{equation}
\label{eqA:ADMM_GD_Square}
C \big(1-\tau_G^\star\big) 
\le 
\big(1-\tau_A^\star\big)^2 \le
2 \Delta C \big(1-\tau^\star_G\big),
\end{equation}
where $\Delta = d_{\textnormal{max}}/d_{\textnormal{min}}$ is the ratio of
the maximum to the minimum degree of $\G$.
Here $\delta = 1 - \omega^\star$ is the spectral gap.
\end{theorem}
\begin{proof}
Using the bounds 
\eqref{eq:bounds} into \eqref{eq:tau_g_star}
we have
\begin{equation}
\label{eq:TauG_Delta}
\tau^\star_G \geq 
\dfrac{d_{\max} - \lambda_{|\E|-1}(\L)}{d_{\max} + \lambda_{|\E|-1}(\L)} \geq
\dfrac{d_{\max} - d_{\max}\lambda_{|\E|-1}(\LL)}{
d_{\max} + d_{\max}\lambda_{|\E|-1}(\LL)} = 
\dfrac{\lambda_2(\W)}{2 - \lambda_2(\W)} = 
\dfrac{1 - \delta}{1 + \delta}
\end{equation}
where we used $\lambda_2(\W) = 1 - \lambda_{|\E|-1}(\L)$.
Analogously, we also have the following upper bound:
\begin{equation}
\label{eq:TauG_Delta2}
\tau^\star_G \le 
\dfrac{ 
2d_{\textnormal{max}} - \lambda_{|\E|-1}(\L) }{  
2d_{\textnormal{max}} + \lambda_{|\E|-1}(\L) } 
\le
\dfrac{ 
2d_{\textnormal{max}} - d_{\textnormal{min}}\lambda_{|\E|-1}(\LL) }{  
2d_{\textnormal{max}} + d_{\textnormal{min}}\lambda_{|\E|-1}(\LL) } 
= \dfrac{2 \Delta - \delta}{2 \Delta + \delta}
\end{equation}
where we defined $\Delta \equiv 
d_{\textnormal{max}}/d_{\textnormal{min}} \ge 1$.

Let us consider the leading order behaviour of $\tau_A^\star$.
Writing in terms of the spectral gap $\delta = 1 - \lambda_2(\W)$,
from 
\eqref{eqA:optimal_gamma_rho_ADMM} and
\eqref{eqA:optimal_rate_ADMM} we have
\begin{align}
\label{eq:TauA_Delta0}
\tau_A^\star &= 1 - \sqrt{2\delta} + 2 \delta +  \mathcal{O}(\delta^{3/2}), \\
(1-\tau_A^\star)^2 &= 2\delta\big(
1 - 2\sqrt{2\delta} + \mathcal{O}(\delta)
\big).
\label{eq:TauA_Delta}
\end{align}
From inequality \eqref{eq:TauG_Delta} we obtain
\begin{equation}
1-\tau_G^{\star} \le \dfrac{2\delta}{1+\delta} \le 2\delta
\end{equation}
Using this into \eqref{eq:TauA_Delta} we obtain
the lower bound
\begin{equation}
(1-\tau_A^\star)^2  
\ge (1-\tau_G^\star)(1  - \mathcal{O}\big(\sqrt{\delta})\big)
\end{equation}
which is
conjecture \eqref{eq:conjecture}.
Analogously, from \eqref{eq:TauG_Delta2} obtain
\begin{equation} 
2\delta \le
2\Delta(1-\tau_G^\star)\left(1+\tfrac{\delta}{2\Delta}\right),
\end{equation}
which replaced into \eqref{eq:TauA_Delta}
gives 
\begin{equation}
(1-\tau_A^\star)^2 \le 2\Delta (1-\tau_G^\star)\big( 1 -
\mathcal{O}\big(\sqrt{\delta}\big) \big)
\end{equation}
and the proof is complete.
\end{proof}


\end{document}